\def\draft{0}

\RequirePackage[l2tabu,orthodox]{nag}
\documentclass
[11pt,letterpaper]
{article}

\usepackage[utf8]{inputenc}
\usepackage{etex}
\usepackage{xspace,enumerate}
\usepackage[dvipsnames]{xcolor}
\usepackage[T1]{fontenc}
\usepackage[full]{textcomp}
\usepackage[american]{babel}
\usepackage{mathtools}
\usepackage{titling}
\usepackage{amsthm}
\usepackage{thmtools,thm-restate}

\usepackage[
letterpaper,
top=1.2in,
bottom=1.2in,
left=1in,
right=1in]{geometry}
\usepackage{newpxtext} 
\usepackage{textcomp} 
\usepackage[varg,bigdelims]{newpxmath}
\usepackage{bm} 
\let\mathbb\varmathbb
\usepackage{microtype}
\usepackage{paralist}
\usepackage{turnstile}
\usepackage{mdframed}
\usepackage{tikz}
\usepackage{caption}
\DeclareCaptionType{Algorithm}
\usepackage{newfloat}
\usepackage[
pagebackref,
colorlinks=true,
urlcolor=blue,
linkcolor=blue,
citecolor=OliveGreen,
]{hyperref}
\usepackage[capitalise,nameinlink]{cleveref}

\newtheorem{theorem}{Theorem}[section]
\newtheorem*{theorem*}{Theorem}

\newtheorem{proposition}[theorem]{Proposition}
\newtheorem*{proposition*}{Proposition}
\newtheorem{lemma}[theorem]{Lemma}
\newtheorem*{lemma*}{Lemma}
\newtheorem{corollary}[theorem]{Corollary}
\newtheorem*{conjecture*}{Conjecture}
\newtheorem{fact}[theorem]{Fact}
\newtheorem*{fact*}{Fact}

\newtheorem*{hypothesis*}{Hypothesis}

\theoremstyle{definition}
\newtheorem{definition}[theorem]{Definition}
\newtheorem*{definition*}{Definition}

\newtheorem*{question*}{Question}

\newtheorem{algorithm}[theorem]{Algorithm}

\newtheorem*{remark*}{Remark}
\theoremstyle{remark}

\newtheorem*{claim*}{Claim}
\newtheorem{remark}[theorem]{Remark}

\newtheorem*{observation*}{Observation}

\crefname{lemma}{Lemma}{Lemmas}
\crefname{fact}{Fact}{Facts}
\crefname{theorem}{Theorem}{Theorems}
\crefname{corollary}{Corollary}{Corollaries}
\crefname{claim}{Claim}{Claims}
\crefname{example}{Example}{Examples}
\crefname{algorithm}{Algorithm}{Algorithms}
\crefname{problem}{Problem}{Problems}
\crefname{definition}{Definition}{Definitions}

\ifnum\draft=1
\newcommand{\Authornotev}[2]{\textcolor{violet}{[#1:} \textcolor{violet}{#2]}}
\newcommand{\Authornoteb}[2]{\textcolor{blue}{[#1:} \textcolor{blue}{#2]}}
\newcommand{\Authornoter}[2]{\textcolor{red}{[#1:} \textcolor{red}{#2]}}
\else
\newcommand{\Authornotev}[2]{}
\newcommand{\Authornoteb}[2]{}
\newcommand{\Authornoter}[2]{}
\fi

\newcommand{\Authornotecolored}[3]{}
\newcommand{\Authorcomment}[2]{}
\newcommand{\Authorfnote}[2]{}

\definecolor{forestgreen(traditional)}{rgb}{0.0, 0.27, 0.13}

\usepackage{boxedminipage}

\newcommand{\Paren}[1]{\left(#1\right)}

\newcommand{\Abs}[1]{\left\lvert#1\right\rvert}

\newcommand{\Set}[1]{\left\{#1\right\}}

\newcommand{\norm}[1]{\lVert#1\rVert}
\newcommand{\Norm}[1]{\left\lVert#1\right\rVert}

\newcommand{\iprod}[1]{\langle#1\rangle}

\newcommand{\Esymb}{\mathbb{E}}
\newcommand{\Psymb}{\mathbb{P}}

\DeclareMathOperator*{\E}{\Esymb}

\DeclareMathOperator*{\ProbOp}{\Psymb}
\renewcommand{\Pr}{\ProbOp}

\newcommand{\mper}{\,.}
\newcommand{\mcom}{\,,}
\newcommand\bdot\bullet

\DeclareMathOperator{\poly}{poly}

\newcommand{\Z}{\mathbb Z}
\newcommand{\N}{\mathbb N}
\newcommand{\R}{\mathbb R}

\newcommand{\cA}{\mathcal A}

\newcommand{\cC}{\mathcal C}
\newcommand{\cD}{\mathcal D}

\newcommand{\cM}{\mathcal M}
\newcommand{\cN}{\mathcal N}
\newcommand{\cO}{\mathcal O}

\newcommand{\cY}{\mathcal Y}

\newcommand{\bbQ}{\mathbb Q}

\newcommand{\ind}{\mathbf{1}}

\renewcommand{\leq}{\leqslant}
\renewcommand{\le}{\leqslant}
\renewcommand{\geq}{\geqslant}
\renewcommand{\ge}{\geqslant}
\let\epsilon=\varepsilon
\numberwithin{equation}{section}
\newcommand\MYcurrentlabel{xxx}
\newcommand{\MYstore}[2]{%
  \global\expandafter \def \csname MYMEMORY #1 \endcsname{#2}%
}
\newcommand{\MYload}[1]{%
  \csname MYMEMORY #1 \endcsname%
}
\newcommand{\MYnewlabel}[1]{%
  \renewcommand\MYcurrentlabel{#1}%
  \MYoldlabel{#1}%
}
\newcommand{\MYdummylabel}[1]{}
\newcommand{\torestate}[1]{%
  \let\MYoldlabel\label%
  \let\label\MYnewlabel%
  #1%
  \MYstore{\MYcurrentlabel}{#1}%
  \let\label\MYoldlabel%
}
\newcommand{\restatetheorem}[1]{%
  \let\MYoldlabel\label
  \let\label\MYdummylabel
  \begin{theorem*}[Restatement of \cref{#1}]
    \MYload{#1}
  \end{theorem*}
  \let\label\MYoldlabel
}
\newcommand{\restatelemma}[1]{%
  \let\MYoldlabel\label
  \let\label\MYdummylabel
  \begin{lemma*}[Restatement of \cref{#1}]
    \MYload{#1}
  \end{lemma*}
  \let\label\MYoldlabel
}
\newcommand{\restateprop}[1]{%
  \let\MYoldlabel\label
  \let\label\MYdummylabel
  \begin{proposition*}[Restatement of \cref{#1}]
    \MYload{#1}
  \end{proposition*}
  \let\label\MYoldlabel
}
\newcommand{\restatefact}[1]{%
  \let\MYoldlabel\label
  \let\label\MYdummylabel
  \begin{fact*}[Restatement of \prettyref{#1}]
    \MYload{#1}
  \end{fact*}
  \let\label\MYoldlabel
}
\newcommand{\restate}[1]{%
  \let\MYoldlabel\label
  \let\label\MYdummylabel
  \MYload{#1}
  \let\label\MYoldlabel
}

\newcommand{\eps}{\epsilon}

\allowdisplaybreaks
\DeclareMathOperator{\pE}{\widetilde{\mathbb{E}}}

\def\norm#1{\left\| #1 \right\|}

\def \dtv{d_{\mathsf{TV}}}

\DeclareMathOperator{\EX}{\mathbb{E}}

\def\tzeta{\tilde{\zeta}}

\newcommand*{\threefrac}[3]{%
  \ensuremath{%
    \vcenter{%
      \halign{\hfil$\,##\,$\hfil\cr
        \scriptstyle{#1}\cr
        \noalign{\kern\threefracLineSep}%
        \hline
        \noalign{\kern\threefracLineSep}%
        \scriptstyle{#2}\cr
        \noalign{\kern\threefracLineSep}%
        \hline
        \noalign{\kern\threefracLineSep}%
        \scriptstyle{#3}\cr
      }%
    }%
  }%
}
\newcommand*{\threefracLineSep}{.4ex}

\newcommand{\Lap}{\mathrm{Lap}}
\newcommand{\tLap}{\mathrm{tLap}}

\usepackage{enumitem}

\usepackage{titling}

\newcommand{\avnote}{\Authornoteb{Ameya}}
\newcommand{\pasin}{\Authornotev{Pasin}}

\newcommand{\stab}{\mathrm{Stab}}
\newcommand{\score}{\mathrm{score}}
\renewcommand{\stab}{\mathrm{stab}} 
\newcommand{\Pot}{\mathrm{Pot}} 
\newcommand{\Alg}{\textsf{Alg}}
\newcommand{\SelectionAlg}{\textrm{Selection }} 
\DeclareMathOperator{\vectorize}{\mathrm{vec}}

\title{
  Private Robust Estimation by Stabilizing Convex Relaxations
}

\author{
    Pravesh K. Kothari \thanks{Carnegie Mellon University} \\praveshk@cs.cmu.edu \and Pasin Manurangsi\thanks{Google Research} \\ pasin@google.com  \and Ameya Velingker \thanksmark{2} \\ameyav@google.com
}

\begin{document}
\maketitle
\begin{abstract}
We give the first polynomial time and sample $(\epsilon,\delta)$-differentially private (DP) algorithm to estimate the mean, covariance and higher moments in the presence of a constant fraction of adversarial outliers. Our algorithm succeeds for families of distributions that satisfy two well-studied properties in prior works on robust estimation: \emph{certifiable subgaussianity} of directional moments and \emph{certifiable hypercontractivity} of degree 2 polynomials. Our recovery guarantees hold in the ``right affine-invariant norms'': Mahalanobis distance for mean, multiplicative spectral and relative Frobenius distance guarantees for covariance  and injective norms for higher moments. Prior works obtained private robust algorithms for mean estimation of subgaussian distributions with bounded covariance. For covariance estimation, ours is the  first efficient algorithm (even in the absence of outliers) that succeeds without any condition-number assumptions. 

Our algorithms arise from a new framework that provides a general blueprint for modifying convex relaxations for robust estimation to satisfy strong \emph{worst-case} \emph{stability} guarantees in the \emph{appropriate parameter norms} whenever the algorithms produce \emph{witnesses of correctness} in their run. We verify such guarantees for a modification of standard sum-of-squares (SoS) semidefinite programming relaxations for robust estimation. Our privacy guarantees are obtained by combining stability guarantees with a new ``estimate dependent'' noise injection mechanism in which noise scales with the eigenvalues of the estimated covariance. We believe this framework will be useful more generally in obtaining DP counterparts of robust estimators.

Independently of our work, Ashtiani and Liaw~\cite{arxiv-AshtianiL21} also obtained a polynomial time and sample private robust estimation algorithm for Gaussian distributions.

\end{abstract}
\newpage

\section{Introduction}
In this work, we consider the problem of efficiently estimating the mean, covariance and, more generally, the higher moments of an unknown high-dimensional probability distribution on $\R^d$, given a sample $y_1, y_2, \ldots, y_n \in \R^d$, under two design constraints: outlier robustness and privacy. The first demands that we build estimators for such basic parameters of probability distributions that tolerate a fixed (dimension-independent) constant fraction of adversarial outliers in the input data. The second demands that our estimators preserve the privacy of individual points $y_i$s (that we model as being contributed by different individuals) participating in our input data. 

Sans privacy constraints, the problem of robustly estimating the basic parameters of an unknown distribution has been the focus of intense research in algorithmic robust statistics starting with the pioneering works of~\cite{DKKLMS16,LaiRV16} from 2016. In addition to new (and often, information-theoretically optimal) algorithms  for several basic robust estimation tasks~\cite{KS17,KStein17,HopkinsL18,BK20,DiakonikolasHKK20}, this line of work has led to a deeper understanding of the properties of the underlying distribution (algorithmic certificates of analytic properties such as subgaussianity, hypercontractivity and anti-concentration, resilience~\cite{SteinhardtCV18}) that make robust estimation possible along with general frameworks such as outlier filtering and the sum-of-squares (SoS) method for attacking algorithmic problems in robust statistics. 

Sans outlier robustness constraints, the task of \emph{private estimation} of the mean and covariance of probability distributions has also seen considerable progress in the recent years. \emph{Differential privacy}~\cite{DworkMNS06} has emerged as a widely-used standard for providing strong individual privacy guarantees. Under differential privacy, a single sample is not allowed to have too significant of an impact on the output distribution of an algorithm that operates on a dataset. Differential privacy has now been deployed in a number of production systems, including those at Google~\cite{ErlingssonPK14,BittauEMMRLRKTS17}, Microsoft~\cite{DingKY17}, Apple~\cite{Greenberg2016Apple, AppleDP2017}, and the US Census Bureau~\cite{Abowd18}. While initial approaches to estimating the mean and covariance under differential privacy required \emph{a priori} bounds on the support of the samples, a more recent work~\cite{KarwaV18} managed to obtain the first private mean estimation algorithm for samples with unbounded support. Subsequent works have built on this progress to obtain differentially private algorithms for mean estimation and covariance estimation (under assumptions on the condition number of the unknown covariance) of Gaussian and heavy-tailed distributions~\cite{KamathLSU19,BunS19,Bun0SW19,CaiWZ19,BiswasDKU20,KamathSU20,DuFMBG20,WangXDX20,Aden-AliAK21,BrownGSUZ21}. 

In this paper, we focus on the task of finding efficient estimation algorithms for mean, covariance and, more generally, higher moments with recovery guarantees in multiplicative spectral distance (i.e., an affine invariant guarantee necessary, for example, to whiten the data or put a set of points in approximate isotropic position) and relative Frobenius distance (necessary for obtaining total variation close estimates of an unknown high-dimensional Gaussian). A very recent work of Liu, Kong, Kakade and Oh~\cite{LiuKKO21} found the first private and robust algorithm for mean estimation under natural distributional assumptions with bounded covariance. However, their techniques do not appear to extend to covariance estimation. Informally, this is because in order to obtain privacy guarantees, we need robust estimation algorithms that are \emph{stable}, i.e., whose output suffers from a bounded perturbation when a single data point is changed arbitrarily. When the unknown covariance is bounded, one can effectively assume that the change in a single data point is bounded. However, in general, the covariance of the unknown distribution can be exponentially (in the underlying dimension) varying eigenvalues which precludes such a method (even in the outlier-free regime). 

\paragraph{This work} In this paper, we give the first algorithms for differentially private robust moment estimation with polynomial time and sample complexity. Our algorithms, in fact, provide a general blueprint for transforming any robust estimation algorithm into a differentially private robust moment estimation algorithm with similar accuracy guarantees as long as the robust estimation algorithm satisfies two key properties: 1) the algorithm is \emph{``witness-producing,''} i.e., the algorithm finds a sequence of ``weights'' on the input corrupted sample that induce a distribution with a relevant property of the unknown distribution family (such as certifiable subgaussianity or hypercontractivity) and 2) the algorithm allows for finding weights that minimize a natural strongly convex objective function in polynomial time. Such properties are naturally satisfied by robust estimation algorithms based on sum-of-squares semidefinite programs. Our main technical result is a simple framework that transforms such an algorithm into one that satisfies \emph{worst-case stability} under input perturbation \emph{in the relevant norms on the parameters}. The final ingredient in our framework is a new noise injection mechanism that uses the stability guarantees so obtained to derive privacy guarantees. This mechanism allows obtaining privacy guarantees even though the distribution of the noise being added depends on the unknown quantity being estimated. In particular, such a subroutine allows us to obtain private robust covariance estimation without any assumptions on the condition number. We note that even without the robustness constraints, a private covariance estimation algorithm without any assumptions on the condition number was not known prior to our work. 

\paragraph{Robustness implies privacy?} Our blueprint presents an intuitively appealing picture---that robustness, when obtained by estimators that satisfy some additional but generic conditions, implies privacy via a generic transformation. This connection might even appear natural: privacy follows by ``adding noise'' to the estimates obtained via algorithms that are insensitive or stable with respect to changing any single point in the input, while robustness involves finding estimators that are \emph{insensitive} to the effects of even up to a constant fraction of outliers. Despite this apparent similarity, there are two key differences that prevent such an immediate connection from being true: 1) privacy is a worst-case guarantee while robustness guarantees are only sensible under distributional assumptions, and, 2) privacy guarantees need insensitivity even against ``inliers.'' Nevertheless, our main result shows that robustness, when obtained via algorithms that satisfy some natural additional conditions, does yield stable (or insensitive) algorithms as required for obtaining differentially private algorithms.

In what follows, we describe our results and techniques in more detail.

\subsection{Our Results}

Formally, our results provide differentially private robust estimation algorithms in the strong contamination model, which we define below. 

\begin{definition}[Strong Contamination Model]
Let $\eta>0$ be the \emph{outlier rate}. Given a distribution $D$ on $\R^d$ and a parameter $n \in \N$, the strong contamination model with outlier rate $\eta$ gives access to a set $Y \subseteq \R^d$ of $n$ points generated as follows: 1) Generate $X \subseteq \R^d$, an i.i.d. sample from $D$ of size $n$, 2) Return any (potentially adversarially chosen) $Y$ such that $|Y \cap X| \geq (1-\eta)n$. In this case, we say that $Y$ is an \emph{$\eta$-corruption} of $X$.
\end{definition}
In the context of analyzing privacy, we will say that two subsets of $n$ points $Y,Y' \subseteq \R^d$ (a.k.a. \emph{databases}) are \emph{adjacent} if they differ in exactly one point (i.e $|Y \cap Y'| \geq n-1$.) We now present our main theorem, which provides a differentially private robust algorithm for moment estimation of an unknown certifiably subgaussian distribution in the strong contamination model. 

Our formal guarantees hold for moment estimation of certifiably subgaussian distributions. A distribution $D$ is \emph{$C$-subgaussian} if for any direction $v$ and any $t \in \N$, $\E_{D} \iprod{x-\mu(D),v}^{2t} \leq (Ct)^t (\E_{D} \iprod{x-\mu(D),v}^2)^t$ where $\mu(D)$ is the mean of the distribution $D$. Certifiable subgaussianity is a stricter version of such a property that additionally demands that the difference between the two sides of the inequality be a sum-of-squares (SoS) polynomial in the variable $v$. Gaussian distributions, uniform distributions on product domains, all strongly log-concave distributions and, more generally, any distribution that satisfies a Poincaré inequality with a dimension-independent constant~\cite{KStein17} are known to satisfy certifiable subgaussianity. See Definition~\ref{def:cert-subgaussianity} and the preliminaries for a detailed discussion.

Our first result is an algorithm for moment estimation of certifiably subgaussian distributions that runs in polynomial time and has polynomial sample complexity.

\begin{theorem} \label{thm:diff-priv-robust-moment-samples}
Fix $C_0 > 0$ and $k\in\N$. Then, there exists an $\eta_0 > 0$ such that for any given outlier rate $0 < \eta \leq \eta_0$ and $\epsilon, \delta > 0$, there exists a randomized algorithm $\Alg$ that takes an input of $n\geq n_0 = \widetilde{\Omega}\left(\frac{d^{4k}}{\eta^2} \left(1 + \left(\frac{\ln(1/\delta)}{\epsilon}\right)^4 + \left(\frac{\ln(1/\delta)}{\epsilon}\right)^{\frac{2k}{k-1}}\right) \cdot C^{4k} k^{4k+6} \right)$ points $Y = \{y_1, y_2,\dots, y_n\} \subseteq \bbQ^d$ (where $C = 2C_0 + \frac{3 \ln(3/\delta)}{\epsilon} + \frac{9}{\epsilon} + 1$), runs in time $(Bn)^{O(k)}$ (where $B$ is the bit complexity of the entries of $Y$) and outputs either ``reject'' or estimates $\hat{\mu} \in \bbQ^d$, $\hat{\Sigma} \in \bbQ^d$, and $\hat{M}^{(t)} \in \bbQ^{d\times d \times \cdots \times d}$ (for all even $t < 2k$ such that $t$ divides $2k$) satisfying the following guarantees:
 
 \begin{enumerate}
\item \textbf{Privacy: } $\Alg$ is $(\epsilon,\delta)$-differentially private with respect to the input $Y$, viewed as a $d$-dimensional database of $n$ individuals. 
  
\item \textbf{Utility: } Let $X = \{x_1, x_2, \dots, x_n\}$ be an i.i.d. sample of size $n\geq n_0$ from a certifiably $C_0$-subgaussian distribution $\cD$ with mean $\mu_*$, covariance $\Sigma_* \succeq 2^{-\poly(d)}I$, and moment tensors $M_*^{(t)}$ for $t\geq 2$. If $Y = \{y_1, y_2, \dots, y_n\}$ is an $\eta$-corruption of $X$, then with probability at least $9/10$ over the draw of $X$ and random choices of the algorithm, $\Alg$ does not reject and outputs estimates $\hat{\mu} \in \bbQ^d$, $\hat{\Sigma} \in \bbQ^{d\times d}$, and $\hat{M}^{(t)} \in \bbQ^{d\times d \times \cdots \times d}$ (for all $t < 2k$ such that $t$ divides $2k$) satisfying the following guarantees:
    \[
    \forall u \in \R^d, \text{  } \iprod{\hat{\mu}-\mu_*,u} \leq O(\sqrt{Ck}) \eta^{1-1/2k} \sqrt{ u^{\top} \Sigma_*u}\mcom
    \]
    and,
    \[
     \left(1-O((Ck)^{t/2k}\right) \eta^{1-1/k}) \Sigma_* \preceq \hat{\Sigma} \preceq \left(1+O((Ck)^{t/2k})\eta^{1-1/k}\right)\Sigma_* \mcom
    \]
    and, for all even $t < 2k$ such that $t$ divides $2k$,
    \[
     \left(1-O(Ck) \eta^{1-t/2k}\right) \iprod{u^{\otimes t}, M_*^{(t)}} \leq \iprod{u^{\otimes t}, \hat{M}^{(t)}} \leq \left(1+O(Ck) \eta^{1-t/2k}\right) \iprod{u^{\otimes t}, M_*^{(t)}} \mper
    \]
 \end{enumerate}
\end{theorem}
 In the above and subsequent theorems, we use the $\widetilde{\Omega}$ notation to hide multiplicative logarithmic factors in $d$, $C$, $k$, $1/\eta$, $1/\epsilon$, and $\ln(1/\delta)$.

\paragraph{Discussion} Our algorithm above achieves an error guarantee in the ``right'' affine-invariant norms similar to the robust moment estimation algorithm of~\cite{KS17}. In particular, the error in the mean in any direction scales proportional to the variance of the unknown distribution providing recovery error bounds in the strong ``Mahalanobis error.'' Similarly, the error in the covariance is multiplicative in the Löwner ordering. Our algorithm succeeds in the standard word RAM model of computation. In particular, the lower bound assumption on the eigenvalue of the unknown covariance in the statement above is entirely an artifact of numerical issues. Such an assumption can be removed (and in particular, we can deal with rank deficient covariances) if we assume that the unknown covariance $\Sigma_*$ has rational entries with polynomial bit complexity. We choose to make an assumption on the smallest eigenvalue of $\Sigma_*$ for the sake of simpler exposition.

Our algorithm above is obtained by applying a general blueprint that applies to any robust estimation algorithms that use ``one-shot rounding'' to produce a differentially private version. We explain our general blueprint in more detail in Section~\ref{sec:overview}. 

\paragraph{Applications} Our differentially private moment estimation algorithm immediately allows us to obtain a differentially private mechanism to implement an outlier-robust \emph{method of moments}. This allows us to learn parameters of statistical models that rely on the method of moments, such as mixtures of spherical Gaussians with linearly independent means~\cite{MR3385380-Hsu13} (that rely on decomposing 3rd moments) as well as independent component analysis~\cite{MR2473563-DeLathauwer07} (that relies on decomposing fourth moments). We direct the reader to the work on robust moment estimation that details such applications~\cite{KS17}.

\paragraph{Covariance estimation in relative Frobenius error} The above theorem provides a multiplicative spectral guarantee. Such a guarantee, however, only yields a dimension-dependent bound on the Frobenius norm of the error. While this is provably unavoidable for the class of certifiably subgaussian distributions, recent work~\cite{bakshi2020mixture} showed that for distributions that satisfy the stronger property of having certifiably hypercontractive degree $2$ polynomials (informally speaking, this is the analog of certifiable subgaussianity for moments of degree $2$ polynomials instead of linear polynomials $\iprod{x,v}$ of the random variable $x$), one can obtain a \emph{dimension-independent} bound on the Frobenius estimation error that vanishes as the fraction of outliers tends to zero. Their algorithm relies on rounding an SoS relaxation with a slightly different constraint system. By working with their constraint system and applying our blueprint for obtaining a ``stable'' version, we obtain a version of the above theorem with the stronger Frobenius estimation guarantee (see \cref{thm:diff-priv-robust-moment-estimation-hypercontractive}). 

By combining our privacy analysis above with the recent work that shows that the algorithm in~\cite{bakshi2020mixture} gives optimal estimation error when analyzed for corrupted samples from a Gaussian distribution, we obtain the following stronger guarantees for private mean and covariance estimation for Gaussian distributions. 
\begin{restatable}[Mean and Covariance Estimation for Gaussian Distributions]{theorem}{gaussianest} \label{thm:gaussian-estimation}
Fix $\epsilon,\delta > 0 $. Then, there exists an absolute constant $\eta_0 > 0$ such that for any given outlier rate $0 < \eta \leq \eta_0$, there exists a randomized algorithm $\Alg$ that takes an input of $n \geq n_0 = \widetilde{\Omega}\left(\frac{d^8}{\eta^4} \left(1 + \frac{\ln(1/\delta)}{\epsilon} \right)^4 \right)$ points $Y \subseteq \bbQ^d$, runs in time $(Bn)^{O(1)}$ (where $B$ is the bit complexity of the entries of $Y$) and outputs either ``reject'' or estimates $\hat{\mu}\in\bbQ^d$ and $\hat{\Sigma}\in\bbQ^{d\times d}$ with the following guarantees:
\begin{enumerate}
    \item \textbf{Privacy: } $\Alg$ is $(\epsilon,\delta)$-differentially private with respect to the input $Y$, viewed as a $d$-dimensional database of $n$ individuals. 
    \item \textbf{Utility: } Let $X = \{x_1, x_2, \dots, x_n\}$ be an i.i.d. sample of size $n\geq n_0$ from a Gaussian distribution with mean $\mu_*$ and covariance $\Sigma_* \succeq 2^{-\poly(d)}I$ such that $Y$ is an $\eta$-corruption of $X$. Then, with probability at least $9/10$ over the random choices of the algorithm, $\Alg$ outputs estimates $\hat{\mu}\in\bbQ^d$ and $\hat{\Sigma}\in\bbQ^{d \times d}$ satisfying the following guarantees:
    \[
    \forall u \in \R^d, \text{  } \iprod{\hat{\mu}-\mu_*,u} \leq \widetilde{O}\left(\eta\cdot\frac{\log(1/\delta)}{\epsilon}\right) \sqrt{ u^{\top} \Sigma_*u}\mcom
    \]
    and,
    \[
    \Norm{\Sigma_*^{-1/2} \hat{\Sigma} \Sigma_*^{-1/2}-I}_F \preceq \widetilde{O}\left(\eta\cdot\sqrt{\frac{\log(1/\delta)}{\epsilon}}\right)   \mcom
    \]
    where the $\widetilde{O}$ hides multiplicative logarithmic factors in $1/\eta$. In particular, $\dtv(\cN(\hat{\mu},\hat{\Sigma}), \cN(\mu_*,\Sigma_*)) < \widetilde{O}(\eta \log(1/\delta)/\epsilon)$.
\end{enumerate}
\end{restatable}

\subsection{Related Work}
Since the works of \cite{DKKLMS16,LaiRV16}, there has been a spate of works designing additional robust estimation algorithms for a wide variety of problems, including mean and covariance estimation~\cite{DiakonikolasKKLMS17, DiakonikolasKKLMS17b,ChengDGW19,DongHL19,HopkinsLZ20,Hopkins20,LiY20}, mixture models~\cite{HopkinsL18, KothariSS18, BakshiK20, DiakonikolasHKK20, BakshiDHKKK20}, principal component analysis (PCA)~\cite{KongSKO20, Jambulapati0T20}, etc. (see survey~\cite{DK-survey17} for details on recent advances in robust statistics). Furthermore, the criterion of \emph{reslience} formulated in \cite{SteinhardtCV18} as a sufficient condition for robustly learning a property of a dataset was subsequently generalized in~\cite{ZhuJS19} in order to deal with a more general class of perturbations.

In the setting of high-dimensional parameter estimation, release of statistics can often reveal signficant information about individual data points, which can be problematic in a number of applications in which it is desirable to protect the privacy of individuals while still providing useful aggregate information (e.g., medical data or census data). Attacks exploiting such properties have been investigated in a long line of works~\cite{DinurN03,BunUV14,DworkSSUV15,SteinkeU15,DworkSSU17,ShokriSSS17}. In light of such exploits, there has been much interest in designing statistical algorithms that protect the privacy of individual samples in a dataset.

In the area of differentially privacy, various works have explored private estimation pertaining to Gaussian mixtures~\cite{NissimRS07,KamathSSU19}, identity testing~\cite{CanonneKMUZ20}, Markov random fields~\cite{ZhangKKW20}, etc.

\paragraph{Concurrent related works} The problem of private robust mean and covariance estimation has been the subject of great interest resulting in a few concurrent and independent related works. Kamath, Mouzakis, Singhal, Steinke, and Ullman~\cite{arxiv-KamathMSSU21} give a differentially private (in the outlier-free regime) algorithm for mean and covariance estimation of Gaussian without making condition number assumptions on the covariance. The work of Liu, Kong, and Oh~\cite{arxiv-LiuKO21} gives a statistical feasibility of private robust estimation with optimal sample complexity via a computationally \emph{inefficient} algorithm. Finally, Hopkins, Kamath, and Majid~\cite{arxiv-HopkinsKM21} also use the sum-of-squares semidefinite programs to obtain private mean estimation (in the outlier-free setting) algorithm for bounded covariance distribution with \emph{pure differential privacy}. Our result are most directly related to the work of Asthiani and Liaw~\cite{arxiv-AshtianiL21} that also obtains efficient private and robust mean and covariance estimation for Gaussian distributions. 

\section{Technical Overview}  \label{sec:overview}
In this section, we give a high-level overview of our general blueprint for obtaining differentially private versions of robust estimation algorithms. As a running example, we will focus on the problem of obtaining private and robust mean and covariance estimators. Specifically, our goal is to design an algorithm that takes input consisting of $n$ points, say $Y \subseteq \R^d$, along with an \emph{outlier rate} $\eta$ and returns estimates of the mean and covariance. We would like the algorithm to be $(\epsilon,\delta)$-differentially private for every $Y$ (i.e., a ``worst-case'' guarantee), viewed as a database in which each $d$-dimensional point in $Y$ is contributed by an individual. We would like the outputs of the algorithm to provide faithful estimates whenever $Y$ is an $\eta$-corruption of a i.i.d. sample from a distribution that has $C$-subgaussian fourth moments. 

For the purpose of the first part of this overview, we recommend the reader to ignore the distinction between certifiable subgaussianity and ``vanilla'' subgaussianity. Recall that a distribution $D$ on $\R^d$ has $C$-subgaussian fourth moments if for every $v \in \R^d$, $\E_{x \sim D} \iprod{x-\mu(D), v}^4 \leq 4C (\E_{x\sim D} \iprod{x-\mu(D),v}^2)^2$. It turns out that the uniform distribution on a $O(d^2)$ size i.i.d. sample $X$ from a $C$-subgussian distribution has $2C$-subgaussian fourth moments. 

\paragraph{Stable robust estimation algorithms} In order to design differentially private algorithms, we need to find robust moment estimation algorithms that are \emph{stable}. Specifically, a robust moment estimation algorithm $\Alg$ is stable if the outputs of $\Alg$ on any pair of adjacent inputs $Y,Y'$ (i.e., inputs that differ in at most one point but arbitrarily so) are close. Such a guarantee must hold over \emph{worst-case} pairs $Y,Y'$---in particular, $Y$ may not be obtained by taking an $\eta$-corruption of an i.i.d. sample from a distribution following our assumptions. This presents a problem at the outset as robust moment estimation algorithms are typically analyzed under \emph{distributional assumptions}. The work of ~\cite{LiuKKO21} addresses this issue by ``opening up'' an iterative filter based algorithm for robust moment estimation and effectively making every step of the algorithm stable. 

\subsection{A Prototypical Robust Estimator to Privatize} 
To understand our ideas, it is helpful to work with a ``prototypical'' but inefficient robust estimation algorithm that we can eventually swap with an efficent one. Let us thus start with a simple (but inefficient) robust estimation algorithm that we call $\Alg$ in the discussion below.

\begin{mdframed}
\begin{algorithm}
\begin{description}
\item[Input: ] $Y = \{y_1, y_2,\ldots, y_n\} \subseteq \R^d$ and outlier rate $\eta>0$.
\item[Output: ] Estimates $\hat{\mu},\hat{\Sigma}$ of mean and covariance or ``reject.''
\item[Operation: ]
\leavevmode
\begin{enumerate}
\item Find a \emph{witness} set of $n$ points $X' \subseteq \R^d$ such that the uniform distribution on $X'$ has subgaussian fourth moments and $|Y \cap X'| \geq (1-\eta)n$. Reject if no such $X'$ exists.
\item Return the mean and covariance of $X'$. 
\end{enumerate}
\end{description}
\end{algorithm}
\end{mdframed}

Observe that the property of having subgaussian fourth moments requires verifying an inequality for every $v \in \R^d$, and, in general, there is no efficient (or even sub-exponential time) algorithm known (or expected, modulo the small-set expansion hypothesis) for this problem. Nevertheless, in \cite{KS17} (see Section 2), the authors prove that a variant of the above program (which we discuss this at the end of this overview) produces estimates that are guaranteed to be close to the mean and covariance of $D$ if $Y$ is an $\eta$-corruption of an i.i.d. sample $X$ from $D$. Note that, though inefficient, such a result is sufficient to establish statistical identifiability of mean and covariance of $D$ from $O_\eta(d^2)$ samples. The closeness guarantees in~\cite{KS17} hold from a more general and basic result that is useful to us in this exposition, which we note below:
\begin{fact}[See Section 2 of~\cite{KS17}, Parameter Closeness from Total Variation Closeness]
Suppose $D,D'$ are two distributions such that 1) both have subgaussian fourth moments and 2) the total variation distance between $D,D'$ is at most $\beta$. Then, for every $v \in \R^d$, $\iprod{\mu(D')-\mu(D),v} \leq O(\eta^{3/4}) \sqrt{v^{\top} (\Sigma(D) + \Sigma(D')) v}$ and $v^{\top} (\Sigma(D') - \Sigma(D))v \leq O(\sqrt{\eta}) v^{\top} (\Sigma(D)+\Sigma(D')) v$. \label{fact:param-closeness-KS-intro} We will say that the means (covariances, respectively) of $D,D'$ are close to within $O(\eta^{3/4})$ ($O(\eta^{1/2})$, respectively) in Mahalanobis distance, to summarize such a guarantee.
\end{fact}  
This fact effectively says that if two distributions both have bounded fourth moments and happen to be close in total variation distance, then their parameters (mean and covariance) must be close. In fact, the closeness is in strong \emph{affine-invariant} norms---often called the Mahalanobis distance for mean and covariance. 

\subsection{Robustness Implies Weak Stability of $\Alg$ with a Randomized Outlier Rate}
Let us now consider the stability of the above inefficient algorithm. We are seemingly in trouble at the outset: as written, there must be two adjacent $Y,Y'$ such that $\Alg$ rejects on $Y$ but not on $Y'$. Let us introduce our first simple idea and show how to patch the algorithm to prevent it from displaying such ``drastic'' change in its behavior. 

\paragraph{Randomizing the outlier rate} The following is a simple but useful observation: If $\Alg$ does not reject on input $Y$ with outlier rate $\eta$, then, $\Alg$ must also not reject on $Y'$ outlier rate $\eta+1/n$. To see why, let $X$ be the set of points with subgaussian fourth moments that intersects $Y$ in $(1-\eta)n$ points. Then, since $Y$ and $Y'$ differ in at most one point, $Y'$ must intersect $X$ in at least $(1-\eta)n - 1 = (1- (\eta+1/n))n$ points. Thus, if, instead of a fixed outlier rate $\eta$, we ran $\Alg$ above with an appropriately ``randomized'' outlier rate, we might expect the rejection probabilities of $\Alg$ on $Y,Y'$ to be similar. Such an argument can be made formal with a simple truncated Laplace noise injection procedure. 

\paragraph{Robustness implies weak stability in Mahalanobis norms} We now address the issue of whether the estimates computed on $Y$ and $Y'$ (assuming $\Alg$ does not reject on either of $Y,Y'$) are close. We first observe that the fact that $\Alg$ is outlier-robust already guarantees a \emph{weak stability} property. Specifically, suppose $X,X'$ are the sets of size $n$ generated by $\Alg$ when run on inputs $Y,Y'$. Then, since $Y \cap Y'$ is of size $n-1$, $|X\cap X'| \geq (1-2\eta)n-1$. Next, observe that intersection bound above is equivalent to the uniform distributions on $X,X'$ having a total variation distance of at most $2\eta + 1/n$. Thus, from Fact~\ref{fact:param-closeness-KS-intro}, we know that the parameters of $X,X'$ are $O(\eta^{O(1)})$ close in the relative Mahalanobis distance defined above. Observe that this argument gives stability properties in the \emph{right norms} directly! However, this is a weak stability guarantee since it only provides a fixed constant distance guarantee instead of $o_n(1)$ that one might expect given that $Y$ and $Y'$ differ in at most $1$ out of $n$ points. Nevertheless, our discussion shows that \emph{robustness, via the inefficient algorithm above, immediately implies weak stability.}

\subsection{A Simple Private Robust Mean Estimator from Weak Stability} 
Can we derive private algorithms from the weak stability guarantees? If the unknown covariance happens to be \emph{spherical} (i.e., has all of its eigenvalues equal to each other), then the Mahalanobis distance guarantees are in fact equivalent (up to constant factor scaling) to Euclidean distance guarantees. As a result, simply adding Gaussian noise calibrated to the sensitivity bounds yields a private robust mean estimation algorithm! Indeed, 1) randomizing the outlier rate, 2) working with the SoS relaxation of the above program and 3) adding Gaussian noise to the resulting estimate, immediately yields a simple, straightforward private robust mean estimator that gives essentially optimal sample complexity guarantees (i.e., matching those of the known non-private robust estimators). 

\paragraph{Weak stability is not enough for covariance estimation}
The challenge in using weak stability to obtain private robust covariance estimators arise when the covariance is non-spherical (e.g., is rank deficient or has eigenvalues of vastly different scales), in which case our Mahalanobis or multiplicative spectral stability guarantee does not translate into Euclidean/spectral norm distance guarantees. In particular, if we were to add Gaussian noise, we would end up \emph{scrambling all small eigenvalues up} and end up with no non-trivial recovery guarantee. 

Indeed, the aforementioned challenge necessitates a rethink of noise injection mechanisms for covariance estimation in general---standard noise addition mechanisms do not appear meaningful in faithfully preserving eigenvalues of different scales. Prior works (e.g., \cite{KamathLSU19}) deal with this by iteratively computing some approximate preconditioning matrices. We have not investigated robust variants of their method. We instead explore \emph{one-shot, blackbox} noise injection mechanisms that still provide us the right guarantees for covariance estimation. 

\subsection{Noise Injection in Estimate-Dependent Norms}
If we wanted to faithfully preserve all eigenvalues (of varying scales) of the unknown covariance, a natural mechanism would be to add noise \emph{linearly transformed with respect to the computed estimate}. For example, if $\hat{\Sigma}$ is the computed estimate, we would like to consider the mechanism that returns $\hat{\Sigma} + \hat{\Sigma}^{1/2} Z \hat{\Sigma}^{1/2}$ where $Z$ is a matrix of random Gaussians. The upshot of such a mechanism is that it adds noise that is scaled relative to the eigenvalues of the estimate $\hat{\Sigma}$---directions where $v^{\top}\hat{\Sigma}v$ is small get a smaller additive noise as against directions where the same quadratic form is large. 

However, the distribution of the added noise in this mechanism \emph{depends on the non-privately estimated quantity itself}. Thus, \emph{a priori}, it provides no useful privacy guarantee!

\textbf{Key Observation:} Nevertheless, our main idea to rescue the above plan is to note that the mechanism above does indeed provide meaningful privacy guarantees (by standard computations from the celebrated Gaussian mechanism) if we are able to guarantee that on any adjacent inputs $Y,Y'$, the non-privately computed estimates are $o_n(1)$ close in relative Frobenius distance! This follows from elementary arguments and is presented in Lemmas~\ref{lem:gaussianhockey} and \ref{lem:hockeysticktensored}. 

The observation above crucially needs the distance between covariances (in relative Frobenius norm) to tend to 0 as $n \rightarrow \infty$; in fact, we need the rate to be inverse polynomial to achieve polynomial sample complexity. Our weak stability guarantee above, however, guarantees only a weak $O(\eta^{1/2})$ bound on multiplicative spectral distance which translates into a relative Frobenius bound of $O(\eta^{1/2} \sqrt{d})$---not only does this not tend to 0 as $n \rightarrow \infty$ but it, in fact, explodes as $d \rightarrow \infty$. 

Thus, in order to use the above mechanism for covariance estimation, we must come up with significantly stronger (and asymptotically vanishing) stability guarantees. Let us investigate how to obtain such guarantees next. 

\subsection{Strong Stability for Robust Estimation Algorithms}
\paragraph{Lack of stability because of multiple differing solutions}
There is an important barrier that prevents $\Alg$ from offering the strong stability guarantees we need in the covariance estimation mechanism above. Consider the case when $Y$ is an i.i.d. sample from a one-dimensional standard Gaussian distribution with mean $0$ and variance $1$ without any outliers added to it. Then, $\cN(0,1\pm c\eta)$ for a small enough constant $c$ is $\eta$-close in total variation distance to $\cN(0,1)$. By a straighforward argument, this implies that we can choose $X'$ to be an i.i.d. sample of size $n$ from $\cN(0,1 \pm c\eta)$--- if $n$ is large enough, then $X'$ will have subgaussian fourth moments and will intersect $Y$ in $(1-\eta)n$ points. The two difference distributions (and the corresponding samples $X'$) however, have variances differing by an additive $O(\eta)$---a fixed constant independent of the sample size $n$. This shows that even in one dimension, $\Alg$ has feasible solutions with variance both $(1-O(\eta))$ and $1+O(\eta)$. Observe that this issue concerns the output of $\Alg$ itself, which can belong to a range that is significantly larger than what we can tolerate---we have not yet touched upon the issue of what happens when we change $Y$ to an adjacent $Y'$. 

\paragraph{Convexification and entropy surrogates}
In order to modify $\Alg$ to output a canonical solution (and with an eye for satisfying the stronger stability property), we wish to make the feasible solution space of $\Alg$ belong to a convex set (instead of the discrete set of solutions $X'$ that intersect with $Y$ in $(1-\eta)n$ points). With no fear of computational complexity, this is easy to do in a canonical way: we search instead for a \emph{probability distribution} over $X'$ that satisfy the constraints that $\Alg$ imposes. Unlike $X'$, distributions on $X'$ that satisfy the constraints are easily seen to form a convex set. 

Given such a convex set, we can resolve our difficulty of not having canonical solutions for any given $Y$ by simply finding a solution (i.e., a probability distribution $\zeta$ over $X'$) that minimizes an appropriate strongly convex objective function. Specifically, for any $X'$, let $w_i$ be the $0$-$1$ indicator of those indices $i$ where $x_i = y_i$. Then, the constraints in $\Alg$ force $\sum_i w_i \geq (1-\eta)n$, and the distribution $\zeta$ can be thought to be over $(X',w)$ in a natural way. 

In order to ensure that $\Alg$ finds a canonical solution, a natural idea is to search over distributions $\zeta$ over $(X',w)$ while minimizing some strongly convex function. We choose the simplest: $\Norm{\E_{\zeta}[w]}_2^2$~\footnote{The exponent of the polynomials appearing in our sample complexity bounds improve if we use a strongly convex function with respect to $1$-norm such as $\Norm{x}_q^2$ for $q = 1+1/\log d$. Our interest is in presenting a general ``privatizing'' blueprint so we continue with the simpler choice above in this work.}. We think of this objective as a surrogate for finding ``maximum entropy solutions'' as, when viewing $\E_{\tzeta}[w_i]$ as defining a probability distribution over $y_i$, minimizing the $\ell_2$ norm favors ``spread-out'' or high entropy solutions. Since $\Norm{\E_{\zeta}[w]}_2^2$ is a convex function being minimized over convex set of expectations with respect to $\zeta$, we expect that the minimizing solution $\E_{\zeta_*}[w]$ should be unique. 

This is not immediately true, however, as our $\Alg$ as stated outputs the mean of $X'$ (there could be ``multiple'' $X'$ with the same intersection with $Y$, in principle). 

\paragraph{Modifying the output of $\Alg$}
In order to fit our framework better, we modify the above blueprint in $\Alg$ to instead output the weighted average of points in $Y$ instead of $X'$. While such a procedure is not directly analyzed in~\cite{KS17}, the methods there can be naturally adapted without much hiccup. As a result we obtain the following modified version of $\Alg$ that we can now work with:

\begin{mdframed}
\begin{algorithm}
\begin{description}
\item[Input: ] $Y = \{y_1, y_2,\ldots, y_n\} \subseteq \R^d$ and an outlier rate $\eta>0$.
\item[Output: ] Estimates $\hat{\mu},\hat{\Sigma}$ of mean and covariance or ``reject.''
\item[Operation: ]
\leavevmode
\begin{enumerate}
\item Find a probability distribution $\zeta$ over a \emph{witness} set of $n$ points $X' \subseteq \R^d$ and intersection indicator $w \in \{0,1\}^n$ that minimizes $\Norm{\E_{\zeta}[w]}_2^2$ and is supported on $(X',w)$ such that 1) the uniform distribution on $X'$ has subgaussian fourth moments and 2) $\sum_i w_i  \geq (1-\eta)n$. Reject if no such $\zeta$  exists.
\item Return $\hat{\mu} = \frac{1}{Z} \sum_i \E_{\zeta}[w_i]  y_i$, $\hat{\Sigma} = \frac{1}{Z} \sum_i \E_{\zeta}[w_i]  (y_i-\hat{\mu})(y_i-\hat{\mu})^{\top}$ where $Z = \sum_i \E_{\zeta}[w_i]$. 
\end{enumerate}
\end{description}
\end{algorithm}
\end{mdframed}
With this modification, $\Alg$ outputs a canonical single solution on any given $Y$ (or rejects). 

\paragraph{Stability of $\Alg$ from the stability of the entropy potential}
We now return to the issue of stability. What happens if we switch the input $Y$ of $\Alg$ above to $Y'$? The strongly convex objective we imposed in the above discussion comes in handy here! Namely, by basic convex analysis (see  Proposition~\ref{prop:pythagorean-theorem-convex-set}), it follows that if optimum entropy potential values of $\Alg$ on $Y$ and $Y'$ are say, $O(1)$-close, then, the vectors $\E_{\zeta}[w](Y)$ and $\E_{\zeta}[w](Y')$ are themselves $O(1)$ close. Recall that each $\E_{\zeta}[w_i]$ is a number in $[0,1]$ and that these numbers add up to $1$. Hence, intuitively speaking, $O(1)$-closeness of $\Norm{\E_{\zeta}[w]}_2^2$ corresponds to constant perturbation in a constant number of coordinates. 

Thus, working with the strongly convex objective above reduces our stability analysis of $\Alg$ to simply understanding how much can our entropy potential change when changing a single point in $Y$. 

Unfortunately, this change can be large in general. $\Norm{\E_{\zeta}[w]}_2^2$ varies between $(1-\eta)n$ and $(1-\eta)^2n$. The additive difference between these two extremes is $O(\eta n) \gg O(1)$. 

\paragraph{Stabilizing the entropy potential: private stable selection}
Before describing our key idea, we first make a simple observation: Fix an input $Y$ and consider the optimum value of the entropy potential of $\Alg$ when run with outlier rate $\eta$. What happens if we change $\eta$ to $\eta+1/n$? Clearly, the potential cannot \emph{increase}: any solution $\zeta$ with outlier rate $\eta$ is also a solution for outlier rate $\eta+1/n$. The potential can decrease arbitrarily though. 

More specifically, we show the following: in order to make the entropy potential stable under a change of $Y$ to an adjacent $Y'$, it is enough to run $Y$ with an outlier rate $\eta'=O(\eta)$ such that the entropy potential of $\Alg$ on $Y$ for any outlier rate in the interval $[\eta'-L/n,\eta'+L/n]$ is within an additive $\widetilde{O}(L/n)$ of any other. 

To see why this claim could be true, informally speaking, observe that if $Y'$ is obtained from $Y$ by changing at most a single point, then a solution $\zeta$ with outlier rate $\eta'$ can be modified into a solution $\zeta'$ for $Y'$ with outlier rate $\eta'+1/n$ by simplying zeroing out the $w_i$ for the index $i$ where $Y'$ and $Y$ differ. This allows us to relate the potentials for neighboring outlier rates on $Y$ and $Y'$. Under the above assumption, the potential remains stable in an interval around $\eta'$ on $Y$. This allows us to conclude that the same must be true for $Y'$ for the interval $[\eta'-L/n +1, \eta'+L/n-1]$. 

The above reasoning allows us to obtain strong stability guarantees if we can 1) show that a stable interval as above exists and 2) find such an interval via a stable process. 

\paragraph{A stable selection procedure via the exponential mechanism}
We show that a stable interval as above (for $L=\widetilde{O}_n(1)$) exists via a simple Markov-like argument. Using an appropriate scoring rule, we show that the standard exponential mechanism can then be used to produce a stable interval like above via a stable algorithm (see \cref{sec:approx-dp-selection}). 

\paragraph{Putting things together}
Altogether, we obtain a version of $\Alg$ that outputs a sequence of weights (i.e., $\E_{\zeta}[w_i]$) that are stable under the modification of a single point in $Y$. When viewed as a distribution on $Y$, the stability guarantee we obtain corresponds to an $\ell_1$-stability of $\widetilde{O}(1/\sqrt{n})$ compared to the $O(\eta)$ (a fixed constant) stability that follows from any naive robust estimation algorithm. 

We note that $\widetilde{O}(1/\sqrt{n})$ can be upgraded to $\widetilde{O}(1/n)$ if we work with a more sophisticated potential function $\Norm{x}_q^2$ for $q = 1+1/\log n$. 

By applying Fact~\ref{fact:param-closeness-KS-intro}, we immediately get that if $\Alg$ does not reject on $Y,Y'$, then the parameters of the respective inputs must be close in the Mahalanobis distance up to a \emph{polynomially vanishing function} of $n$, as desired. This allows us to implement the estimate-dependent noise injection mechanism for covariance estimation! 

We note that the discussion above can be formalized into an \emph{information-theoretic private identifiability algorithm} (i.e., an inefficient private robust algorithm). We next discuss how to transform the above blueprint result into an efficient algorithm. 

\subsection{From Ideal Algorithms to Efficient Algorithms}
Let us now go back and summarize 1) facts about the idealized inefficient algorithm and 2) our general blueprint for making such an algorithm $\Alg$ private. 

\begin{enumerate}
	\item \textbf{Witness Production:} We have used that the fact that $\Alg$ searches over witnesses $X'$ that share the relevant property of the distributional model we have chosen (e.g., subgaussianity of fourth moments in the above discussion). 
	\item \textbf{Strongly Convex Entropy Potential:} We have minimized a strongly convex potential function in order to ensure that $\Alg$ outputs a canonical solution. 
	\item \textbf{Stable Outlier Rate Selection:} We have implemented a randomized stable selection scheme (via the exponential mechanism) for the outlier rate in order to argue that the optimum entropy potential of $\Alg$ is stable under the modification of a single point in the input $Y$. 
\end{enumerate}

We can apply this scheme to any algorithm that outputs a sequence of weights on the input sample $Y$, subject to the constraint that 1) the weights induce the relevant property of the distributional model, and 2) they minimize a strongly convex potential function. 

\paragraph{Witness-producing SoS-based robust estimation algorithms}
It turns out that we can ensure all the above properties for \emph{efficient} robust estimation algorithms based on ``one-shot rounding'' of convex relaxations. We specifically rely on the algorithms for robust estimation based on SoS semidefinite programs in this work. 

The SoS-based algorithms in the prior works that we use ~\cite{BK20,KS17} almost fit our requirements except with two technical constraints: 
\begin{enumerate}
\item The algorithms in the aforementioned prior works do not output weights on $Y$ explicitly. However, we are able to show that a natural modification that outputs such weights on $Y$ can be analyzed by the same methods.
\item The algorithms in the aforementioned prior works were analyzed under distributional assumptions on $Y$ without the need to explicitly argue that the weights induce good witnesses (which we desire in our above analysis). Indeed, arguing that these algorithms produce such witnesses on worst-case datasets $Y$ (whenever they don't reject) appears challenging. However, we are able to get by without such a statement by observing that we can adapt the analyses of the algorithms in the prior works to infer the following statement: if the algorithm returns a good witness on $Y$, then under a small perturbation of the parameters, it must also return a good witness on an adjacent $Y'$. 
\end{enumerate}

While verifying the properties makes our transformation not entirely blackbox at the moment, we strongly believe that our blueprint demonstrates a conceptually appealing connection between robust algorithm design and private algorithm design. Concretly, we expect our blueprint to be useful in designing more private (and robust) estimation algorithms. Indeed, we believe our techniques immediately extend to other problems where SoS-based robust estimation algorithms are known, such as linear regression~\cite{KlivansKM18,bakshi2020robust} and clustering spherical and non-spherical mixtures~\cite{DiakonikolasHKK20,BK20,HopkinsL18,KS17a,TCS-086}.

\section{Preliminaries}

In this work, we will deal with algorithms that operate on numerical inputs. In all such cases, we will rely on the standard word RAM model of computation and assume that all the numbers are rational represented as a pair of integers describing the numerator and the denominator. In order to measure the running time of our algorithms, we will need to account for the length of the numbers that arise during the run of the algorithm. The following definition captures the size of the representations of rational numbers:

\begin{definition}[Bit Complexity]
The bit complexity of an integer $p \in \Z$ is $1+ \lceil \log_2 p \rceil$. The bit complexity of a rational number $p/q$ where $p,q \in \Z$ is the sum of the bit complexities of $p$ and $q$. 
\end{definition}

For any finite set $X$ of points in $\R^d$, we will use $\mu(X),\Sigma(X), M^{(t)}(X)$ to denote the mean, covariance and the $t$-th moment tensor of the uniform distribution on $X$. 

\subsection{Pseudo-Distributions}
Pseudo-distributions are generalizations of probability distributions and form dual objects to sum-of-squares proofs in a precise sense that we will describe below.

\begin{definition}[Pseudo-distribution, Pseudo-expectations, Pseudo-moments] \label{def:pseudoexpectation}
A \emph{degree-$\ell$ pseudo-distribution} is a finitely-supported function $D:\R^n \rightarrow \R$ such that $\sum_{x} D(x) = 1$ and $\sum_{x} D(x) f(x)^2 \geq 0$ for every polynomial $f$ of degree at most $\ell/2$. (Here, the summations are over the support of $\mu$.) 

The \emph{pseudo-expectation} of a function $f$ on $\R^d$ with respect to a pseudo-distribution $D$, denoted $\pE_{D(x)} f(x)$, as
\begin{equation}
  \pE_{D(x)} f(x) = \sum_{x} D(x) f(x) \,\mper
\end{equation}
In particular, the \emph{mean} $\mu$ of a pseduo-distribution is defined naturally as the pseudo-expectation of $f(x) = x$, i.e., $\mu\pE_{D(x)} x$.

The degree-$\ell$ moment tensor of a pseudo-distribution $\mu$ is the tensor $\E_{\mu(x)} (1,x_1, x_2,\ldots, x_n)^{\otimes \ell}$.
In particular, the moment tensor has an entry corresponding to the pseudo-expectation of every monomial of degree at most $\ell$ in $x$.
\end{definition}

Observe that if a pseudo-distribution $\mu$ satisfies, in addition, that $\mu(x) \geq 0$ for every $x$, then it is a mass function of some probability distribution. Further, a straightforward polynomial-interpolation argument shows that every degree-$\infty$ pseudo-distribution satisfies $\mu \ge 0$ and is thus an actual probability distribution. The set of all degree-$\ell$ moment tensors of probability distribution is a convex set. Similarly, the set of all degree-$\ell$ moment tensors of degree-$d$ pseudo-distributions is also convex.

We now define what it means for $\pE$ to (approximately) satisfy constraints.
\begin{definition}[Satisfying constraints]
For a polynomial $g$, we say that a degree-$k$ $\pE$ satisfies the constraint $\{g = 0\}$ exactly if for every polynomial $p$ of degree $\leq k-\deg(g)$, $\pE[p g] = 0$ and $\tau$-approximately if $|\pE[p g_j]| \leq \tau \norm{p}_2$. We say that $\pE$ satisfies the constraint $\{g \geq 0\}$ exactly if for every polynomial $p$ of degree $\leq k/2 - \deg(g)/2$, it holds that $\pE[p^2 g] \geq 0$ and $\tau$-approximately if $\pE[p^2 g] \geq -\tau \norm{p}_2^2$.  
\end{definition}

The following fact describes the precise sense in which pseudo-distributions are duals to sum-of-squares proofs. 

\begin{fact}[Strong Duality,~\cite{MR3441448-Josz16}, see Theorem 3.70 in~\cite{TCS-086} for an exposition]
Let $p_1, p_2, \ldots, p_k$ be real-coefficient polynomials in $x_1, x_2, \ldots, x_n$. 
Suppose there is a degree-$d$ sum-of-squares refutation of the system $\{p_i(x) \geq 0\}_{i \leq k}$.
Then, there is no pseudo-distribution $\mu$ of degree $\geq d$ satisfying $\{p_i(x) \geq 0\}_{i \leq k}$. 
On the other hand, suppose that there is a pseudo-distribution $\mu$ of degree $d$ consistent with $\{p_i(x) \geq 0\}_{i \leq k}$. Suppose further that the set $\{p_1, p_2, \ldots, p_k\}$ contains the quadratic polynomial $R-\sum_i x_i^2$ for some $R > 0$. Then, there is no degree-$d$ sum-of-squares refutation of the system $\{p_i(x) \geq 0\}_{i \leq k}$.
\end{fact}

\paragraph{Basic sum-of-squares (SoS) proofs}
\begin{fact}[Operator norm Bound]
\label{fact:operator_norm}
Let $A$ be a symmetric $d\times d$ matrix with rational entries with numerators and denominators upper-bounded by $2^B$ and $v$ be a vector in $\mathbb{R}^d$. 
Then, for every $\epsilon \geq 0$, 
\[
\sststile{2}{v} \Set{ v^{\top} A v \leq \|A\|_2\|v\|^2_2 + \epsilon}
\]
The total bit complexity of the proof is $\poly(B,d,\log 1/\epsilon)$.
\end{fact}



\begin{fact}[SoS Hölder's Inequality] \label{fact:sos-holder}
Let $f_i,g_i$ for $1 \leq i \leq s$ be indeterminates. 
Let $p$ be an even positive integer. 
Then, 
\[
\sststile{p^2}{f,g} \Set{  \Paren{\frac{1}{s} \sum_{i = 1}^s f_i g_i^{p-1}}^{p} \leq \Paren{\frac{1}{s} \sum_{i = 1}^s f_i^p} \Paren{\frac{1}{s} \sum_{i = 1}^s g_i^p}^{p-1}}\mper
\]
The total bit complexity of the SoS proof is $s^{O(p)}$. 
\end{fact}
Observe that using $p = 2$ yields the SoS Cauchy-Schwarz inequality. 

\begin{fact}[SoS Almost Triangle Inequality] \label{fact:sos-almost-triangle}
Let $f_1, f_2, \ldots, f_r$ be indeterminates. Then,
\[
\sststile{2t}{f_1, f_2,\ldots,f_r} \Set{ \Paren{\sum_{i\leq r} f_i}^{2t} \leq r^{2t-1} \Paren{\sum_{i =1}^r f_i^{2t}}}\mper
\]
The total bit complexity of the SoS proof is $r^{O(t)}$.
\end{fact}

\begin{fact}[SoS AM-GM Inequality, see Appendix A of~\cite{MR3388192-Barak15}] \label{fact:sos-am-gm}
Let $f_1, f_2,\ldots, f_m$ be indeterminates. Then, 
\[
\Set{f_i \geq 0\mid i \leq m} \sststile{m}{f_1, f_2,\ldots, f_m} \Set{ \Paren{\frac{1}{m} \sum_{i =1}^m f_i }^m \geq \Pi_{i \leq m} f_i} \mper
\]
The total bit complexity of the SoS proof is $\exp(O(m))$.
\end{fact}

We will also use the following two consequence of the SoS AM-GM inequality:
\begin{proposition} \label{prop:mixed-vs-pure-monomials}
Let $a,b$ be indeterminates. Then,
\[
\sststile{2t}{a,b} \Set{a^{2j} b^{2t-2j} \leq j a^{2t} + (t-j) b^{2t} }\mper
\]
The total bit complexity of the SoS proof is $\exp(O(t))$. 
\end{proposition}

\begin{proof}
We apply the SoS AM-GM inequality with $f_i = a^2$ for $i=1,\ldots,j$ and $f_i = b^{2}$ for $i = j+1,\ldots,t$. We thus obtain:
\[
\sststile{2t}{a,b} \Set{(j/t a^2 + (1-j/t)b^2)^t \geq a^{2j}b^{2t-2j}}
\]
By the SoS Almost Triangle inequality, we have:
\[
\sststile{2t}{a,b} \Set{(j/t a^2 + (1-j/t)b^2)^t \leq (j a^{2t} + (t-j)b^{2t})}
\]
Combining the above two claims completes the proof. The total bit complexity of the SoS proof follows immediately by using the bounds for the two constituent inequalities used in the proof above. 
\end{proof}

\begin{proposition} \label{prop:even-vs-odd-monomials}
Let $a,b$ be indeterminates. Then, for any positive integers $i,t$ such that $i$ is odd and $2t \geq i$, we have:
\[
\sststile{2t}{a,b} \Set{ a^i b^{2t-i} \leq \frac{1}{2} (a^{i-1} b^{2t-i+1} + a^{i+1}b^{2t-i-1})}\mper
\]
The total bit complexity of the SoS proof is $\exp(O(t))$.
\end{proposition}
\begin{proof}
Write $i = 2r-1$ for some $r \geq 1$. Then, we have: $a^i b^{2t-1} = a^r b^{t-r} a^{r-1} b^{t-r+1}$. 
By the SoS AM-GM inequality with $f_1 = a^r b^{t-r}$ and $f_2 = a^{r-1}b^{t-r+1}$, we thus have:
\[
\sststile{2t}{a,b} \Set{ a^i b^{2t-i} = a^r b^{t-r} a^{r-1} b^{t-r+1}  \leq \frac{1}{2} (a^{i-1} b^{2t-i+1} + a^{i+1}b^{2t-i-1})}\mper
\]
\end{proof}

\begin{fact}[Cancellation within SoS, Constant RHS~\cite{bakshi2020mixture}]
    \label{lem:sos-cancel-basic}
    Suppose $A$ is indeterminate and $t\ge 1$. Then, 
    \[\Set{A^{2t}\le 1}\sststile{2t}{A}\Set{A^2\le 1}\]
    Further, the total bit complexity of the SoS proof is at most $2^{O(t)}$.
\end{fact}

\begin{lemma}[Cancellation within SoS~\cite{bakshi2020mixture}]
\label{lem:sos-cancel}
Suppose $A$ and $C$ are indeterminates and $t\ge 1$. Then,
\[\Set{A\ge 0\cup A^t\le CA^{t-1}}\sststile{2t}{A,C}\Set{A^{2t}\le C^{2t}}.\] Further, the total bit complexity of the SoS proof is at most $2^{O(t)}$.
\end{lemma}

\subsection{Algorithms and Numerical Accuracy} 

The following fact follows by using the ellipsoid algorithm for semidefinite programming. The resulting algorithm to compute  pseudo-distributions approximately satisfying a given set of polynomial constraints is called the \emph{sum-of-squares algorithm}. 

\begin{fact}[Computing pseudo-distributions consistent with a set of constraints \cite{MR939596-Shor87,parrilo2000structured,MR1748764-Nesterov00,MR1846160-Lasserre01}] \label{fact:finding-pseudo-distributions}
There is an algorithm with the following properties: The algorithm takes input $B \in \N$, $\tau >0$, and polynomials $p_1, p_2, \ldots, p_k$ of degree $\ell$ with rational coefficients of bit complexity $B$. If there is a pseudo-distribution of degree $d$ consistent with the constraints $\{p_i(x) \geq 0\}_{i \leq k}$, the algorithm in time $(Bn)^{O(d)} \poly \log (1/\tau)$ outputs a pseudo-distribution $\mu$ of degree $d$ that $\tau$-approximately satisfies $\{p_i(x) \geq 0\}_{i \leq k}$. 
\end{fact}

\newcommand{\Ent}{\mathrm{Ent}}
\newcommand{\Rel}{\mathrm{Rel}}

\subsection{Tensors}
Since we will deal with higher moments of distributions, which are naturally represented as tensors, we will need to define some related notation and conventions for the sake of clarity in our exposition.

Let $[n] = \{1,2,\dots, n\}$ for any natural number $n$. We define the following.
\begin{definition} \label{def:tensormult}
 Suppose we have an $m\times n$ matrix $M$ and an $m'\times n'$ matrix $N$. We define $M\otimes N$ to be the standard $mm' \times nn'$ matrix given by the \emph{Kronecker product} of $M$ and $N$.
 
 Moreover, for an $m\times n$ matrix $M$, we denote by $M^{\otimes t}$ the \emph{$t$-fold Kronecker product} $\underbrace{M\otimes M \otimes \cdots \otimes M}_{t\ \text{times}}$ (of dimension $tm \times tn$).
\end{definition}

Given an $m\times n$ matrix $M$, we will also find it convenient to index $M^{\otimes t}$ as follows: for any $1\leq i_1, i_2, \dots, i_t\leq m$ and $1\leq j_1, j_2, \dots, j_t\leq n$, we can refer to the term $M_{(i_1,i_2,\dots,i_t), (j_1,j_2,\dots,j_t)}^{\otimes t} = \prod_{k=1}^t M_{i_t, j_t}$.

We also define a useful \emph{flattening} operation on tensors:
\begin{definition} \label{def:flattening}
 Given an $m_1\times m_2\times \cdots \times m_t$ tensor $M$, we define the \emph{flattening}, or \emph{vectorization}, of $M$ to be the $(m_1 m_2 \cdots m_t)$-dimensional vector, denoted $\vectorize(M)$, whose entries are precisely the entries of $M$ appearing in the natural lexicographic order on $[m_1]\times [m_2] \times \cdots \times [m_t]$. In other words, the entry $M_{i_1,i_2,\dots,i_t}$ appears before $M_{j_1, j_2, \dots, j_t}$ (where $i_k, j_k \in [m_k]$ for $k=1,2,\dots,t$) in $\vectorize(M)$ if and only if there exists some $1\leq k \leq t$ such that $i_k < j_k$ and $i_l = j_l$ for all $l < k$.
\end{definition}

\begin{definition} \label{def:tensor-innerprod}
 Given an $n$-dimensional vector $u$ and an $\underbrace{n\times n\times \cdots \times n}_{\text{$t$ times}}$-dimensional tensor $M$, we define $\iprod{u^{\otimes t}, M}$ to be $\iprod{\vectorize(u^{\otimes t}), \vectorize(M)}_{\R^d}$, i.e., the value of the standard inner product (on $n^t$-dimensional vectors) between the flattenings of $u^{\otimes t}$ and $M$. 
\end{definition}

A convenient fact we will use is a so-called ``mixed product'' property for matrices.
\begin{fact} \label{fact:mixedprod}
 Given an $m\times n$ matrix $A$, $m'\times n'$ matrix $B$, and $n\times n'$ matrix $V$, we have that
 \[
  AVB^T = (A\otimes B) \vectorize(V),
 \]
where the above is expressed as matrix-vector product.
\end{fact}

Finally, we define the moment tensor for a probability distribution.
\begin{definition}
 Given a probability distribution $\cD$ on $\R^d$ and an integer $t > 1$, we define the \emph{$t^\text{th}$ moment tensor} $M$ to be a $\underbrace{d\times d\times \cdots \times d}_{\text{$t$ times}}$ tensor whose entries are given by $M_{i_1, i_2, \dots, i_t} = \EX_{X\sim\cD}[X_{i_1} X_{i_2} \cdots X_{i_t}]$ for $i_1, i_2, \dots, i_t \in [d]$.
\end{definition}

\subsection{Basic Convexity}
We will use the following basic propositions about convexity in our analysis. 
\begin{proposition}[Neighborhoods of minimizers of convex functions] \label{prop:basic-convexity}
Let $K$ be a closed convex subset of $\R^N$. 
Let $f$ be a smooth convex function on $\R^N$. 
Let $x$ be a minimizer of $f$ on $K$. 
Then, for every $y \in K$, $\iprod{y-x, \nabla f(x)} \geq 0$. 
\end{proposition}

\begin{proof}
If not, then for a small enough positive $\lambda$, $f(x+ \lambda (y-x))<f(x)$. But, $x+ \lambda (y-x) = (1-\lambda)x+\lambda y \in K$. 
\end{proof}


\begin{proposition}[Pythagorean theorem from strong convexity w.r.t 2 norm] \label{prop:pythagorean-theorem-convex-set}
Let $K$ be a convex subset of $\R^d$ for $d \in \N$. Let $x$ be a minimizer of the convex function $f(x)=\Norm{x}_2^2$ on $K$. Let $y \in K$. Then, $f(y)-f(x) \geq \Norm{y-x}_2^2$. 
\end{proposition}
\begin{proof}
We have: $\Norm{y}_2^2 = \Norm{y-x}_2^2 + \Norm{x}_2^2 + 2 \iprod{y-x,x}$. The proposition follows by applying Proposition~\ref{prop:basic-convexity} to observe that $\iprod{y-x,x} \geq 0$. 
\end{proof}




We will also need the following basic bound:
\begin{lemma}
Suppose $x,y \in [0,1]^n$ such that $\sum_i x_i, \sum_i y_i \geq n/2$ and $\Norm{x-y}_1 \leq \beta n$ for $\beta \leq 1/10$. Let $\bar{x} = \frac{x}{\norm{x}_1}$ and $\bar{y} = \frac{y}{\bar{y}_1}$ be the normalized versions of $x,y$. Then, 
\[
\Norm{\bar{x}-\bar{y}}_1 \leq 6 \beta \mper
\]
\label{lem:normalized-vs-unnormalized}
\end{lemma}
\begin{proof}
Suppose, without loss of generality, that $\norm{x}_1 = c_1 n \geq c_2 n = \norm{y}_1$ for $c_1,c_2 \geq 1/2$. Then, we know that $\norm{y}_1 = c_2 n  \geq (c_1 - \beta)n$. 
Thus, $\Norm{\bar{x}-\bar{y}}_1 \leq \frac{1}{c_1c_2 n^2} (\Norm{x \norm{y}_1 - y \norm{x}_1}_1 \leq \frac{1}{c_1c_2 n^2} ( c_1 n \Norm{x-y}_1 + \beta n^2) \leq 6 \beta$. 
\end{proof}

\subsection{Certifiable Subgaussianity}

\begin{definition}[Certifiable Subgaussianity] \label{def:cert-subgaussianity}
A distribution $D$ on $\R^d$ with mean $\mu_*$ is said to be $2k$-certifiably $C$-subgaussian if there is a degree $2k$ sum-of-squares proof of the following polynomial inequality in $d$-dimensional vector-valued indeterminate $v$:
\[
\E_{x \sim D} \iprod{x-\mu_*,v}^{2k} \leq (Ck)^k \Paren{\E_{x \sim D} \iprod{x-\mu_*,v}^2}^k\mper
\]
Furthermore, we say that $D$ is certifiable $C$-subgaussian if it is $2k$-certifiably $C$-subgaussian for every $k\in \N$.

A finite set $X \subseteq \R^d$ is said to be $2k$-certifiable $C$-subgaussian if the uniform distribution on $X$ is $2k$-certifiably $C$-subgaussian. 
\end{definition}

\begin{fact}[Consequence of Theorem 1.2 in~\cite{KS17}] \label{fact:param-closeness-from-tv-closeness}
Let $Y$ be a collection of $n$ points in $\R^d$. Let $p,p' \in [0,1]^n$ be weight vectors satisfying $\norm{p}_1$, $\norm{p'}_1 =1$, and $\norm{p-p'}_1 = \tau$. Suppose that the distributions on $Y$ where the probability of $i$ is $p_i$ ($p_i'$, respectively) is $2k$-certifiably $C_1$ ($C_2$, respectively) subgaussian. Let $\mu_p = \sum_i p_i y_i$, $\Sigma_p = \sum_i p_i (y_i - \mu_p) (y_i-\mu_p)^{\top}$, and $M^{(t)}_p = \sum_i p_i y_i^{\otimes t}$ for every $t \in \N$ be the mean, covariance and $t$-th moment tensor of distribution defined $p$. Define $\mu_{p'}, \Sigma_{p'}, M^{(t)}_{p'}$ similarly for the distribution corresponding to $p'$. 

Then, for every $\tau \leq \eta_0$ for some absolute constant $\eta_0$, for every $u \in \R^d$, $C' = C_1 + C_2$ and $t \leq k$:
\[
\iprod{\mu_p - \mu_{p'},u} \leq \tau^{1-1/2k} \cdot O(\sqrt{C k}) \sqrt{u^{\top} \Sigma_{p} u})\mcom
\]
\[
(1-O(C'k) \tau^{1-1/k}) \Sigma_{p} \preceq \Sigma_{p'} \preceq (1+O(C'k) \tau^{1-1/k}) \Sigma_{p'} \mcom
\]
\[
(1-O(C'^{t/2}k^{t/2}) \tau^{1-t/2k}) \iprod{u^{\otimes t}, M^{(t)}_{p}} \leq \iprod{u^{\otimes t}, \hat{M}^{(t)}_{p'}} \leq \iprod{u^{\otimes t}, M^{(t)}_p} \mcom
\]

\end{fact}

\subsection{Differential Privacy}
\label{subsec:prelim-dp}

In this section, we state a few tools from differential privacy (DP) literature that will be used in our algorithms. We start by recalling the definition of DP:

\begin{definition}[Differential Privacy~\cite{DworkMNS06}]
An algorithm $\cM: \cY \to \cO$ is said to be \emph{$(\eps, \delta)$-differentially private} (or \emph{$(\eps, \delta)$-DP}) for $\eps, \delta > 0$ iff, for every $S \subseteq \cO$ and every neighboring datasets $Y, Y'$, we have
$$\Pr[\cM(Y) \in S] \leq e^\eps \cdot \Pr[\cM(Y') \in S] + \delta.$$
\end{definition}

Throughout this work, our set $Y$ will consist of $y_1, \dots, y_n \in \R^d$. $Y = (y_1, \dots, y_n)$ and $Y' = (y'_1, \dots, y'_n)$ are neighbors iff they differ on a single data point, i.e., $y'_j = y_j$ for all $j \ne i$. Note that this is the so-called \emph{substitution} variant of DP; another popular variant is the \emph{add/remove} DP where a neighboring $Y'$ results from adding or removing an example from $Y$. We remark that it is not hard to extend our algorithm to the add/remove DP setting, by first computing a DP estimate $\hat{n}$ of $n$ and either throwing away random elements or adding zero vectors to arrive at an $n$-size dataset on which our algorithm can be applied. 

\subsubsection{Laplace Mechanism and Its Variants}

The Laplace mechanism~\cite{DworkMNS06} is among the most widely used mechanisms in differential privacy. It works by adding a noise drawn from the Laplace distribution (defined below) to the output of the function one wants to privatize.

\begin{definition}[Laplace Distribution]
The Laplace distribution with mean $\mu$ and parameter $b$ on $\R$, denoted by $\Lap(\mu, b)$, has the PDF $\frac{1}{2b} e^{-|x-\mu|/b}$. 
\end{definition}

We will also use the ``truncated'' version of the Laplace mechanism where the noise distribution is shifted and truncated to be non-negative. The precise definition of the noise distribution and its guarantee is given below. For completeness, we provide the DP analysis (\Cref{lem:tlap-dp}) in \Cref{app:tlap}.

\begin{definition}[Truncated Laplace Distribution]
The (negatively) truncated Laplace distribution with mean $\mu$ and parameter $b$, denoted by $\tLap(\mu, b)$ is defined as $\Lap(\mu, b)$ conditioned on the value being negative.
\end{definition}

\begin{lemma}[Truncated Laplace Mechanism] \label{lem:tlap-dp}
Let $f: \cY \to \R$ be any function with sensitivity at most $\Delta$. Then the algorithm that adds $\tLap\left(-\Delta\left(1 + \frac{\ln\left(1/\delta\right)}{\eps}\right), \Delta/\eps\right)$ to $f$ satisfies $(\eps, \delta)$-DP.
\end{lemma}

Finally, we also state a bound on the tail probability of the truncated Laplace distribution which will be useful in our subsequent analysis.

\begin{lemma} \label{lem:tlap-tail}
 Suppose $\mu < 0$ and $b > 0$. Let $X \sim \tLap(\mu, b)$. Then, for $y < \mu$, we have that
 \[
  \Pr[X < y] = \frac{e^{(y-\mu)/b}}{2 - e^{\mu/b}}.
 \]
\end{lemma}

\subsubsection{Composition Theorem}

It will be convenient to also consider DP algorithms whose privacy guarantee holds only against subsets of inputs. Specifically, we define:
\begin{definition}[Differential Privacy Under Condition] \label{def:dp-cond}
An algorithm $\cM: \cY \to \cO$ is said to be \emph{$(\eps, \delta)$-differentially private under condition $\Psi$} (or \emph{$(\eps, \delta)$-DP under condition $\Psi$}) for $\eps, \delta > 0$ iff, for every $S \subseteq \cO$ and every neighboring datasets $Y, Y'$ both satisfying $\Psi$, we have
$$\Pr[\cM(Y) \in S] \leq e^\eps \cdot \Pr[\cM(Y') \in S] + \delta.$$
\end{definition}

It is not hard to see that an analogue of the basic composition theorem still holds in this setting, which we formalize below. We remark that this is similar to the composition theorem derived in~\cite[Section 5]{DworkL09}. However, since our composition theorem is slightly different, we provide its proof in \Cref{app:composition}.

\begin{lemma}[Composition for Algorithm with Halting] \label{lem:composition}
Let $\cM_1: \cY \to \cO_1 \cup \{\perp\}, \cM_2: \cO_1 \times \cY \to \cO_2 \cup \{\perp\}, \dots, \cM_k: \cO_{k-1} \times \cY \to \cO_k \cup \{\perp\}$ be algorithms. Furthermore, let $\cM$ denote the algorithm that proceeds as follows (with $o_0$ being empty): For $i = 1, \dots, k$, compute $o_i = \cM_i(o_{i - 1}, Y)$ and, if $o_i = \perp$, halt and output $\perp$. Finally, if the algorithm has not halted, then output $o_k$.

Suppose that:
\begin{itemize}
\item For any $1\leq i < k$, we say that $Y$ satisfies the condition $\Psi_i$ if running the algorithm on $Y$ does not result in halting after applying $\cM_1, \cM_2, \dots, \cM_i$.
\item $\cM_1$ is $(\eps_1, \delta_1)$-DP.
\item $\cM_i$ is $(\eps_i, \delta_i)$-DP (with respect to neighboring datasets in the second argument) under condition $\Psi_{i-1}$ for all $i = \{2, \dots, k\}$.
\end{itemize}
Then, $\cM$ is $\left(\sum_{i \in [k]} \eps_i, \sum_{i \in [k]} \delta_i\right)$-DP.
\end{lemma}

\subsubsection{Hockey-Stick Divergence}

It will be convenient in our analysis to use an equivalent definition of DP based on the \emph{hockey-stick divergence}. For ease of notation, let $[a]_+ = \max\{a, 0\}$ for all $a \in \R$.

\begin{definition}[Hockey-Stick Divergence]
Let $p(x),q(x)$ be probability density functions on $\R^d$, and $\alpha$ a non-negative real number. The Hockey-stick divergence $D_\alpha(p,q)$ between $p,q$ is defined as:
\[
D_{e^\epsilon}(p,q) = \int_{x \in \R^d} [p(x)-\alpha \cdot q(x)]_+ dx\mper
\] 
\end{definition}

The following fact is simple to derive from the definition of DP and is often used  in literature.

\pasin{We probably need the fact that $\cA(Y)$ defines a PDF below? I.e. the hockey-stick divergence as defined above is not well-defined for discrete mechanisms. But maybe this is already clear from context?}

\begin{fact}[$(\epsilon,\delta)$-DP from Hockey-Stick Divergence Bounds] \label{fact:dphockey}
Let $\cM: \cY \to \R^d$ be a randomized algorithm. $\cM$ is $(\eps, \delta)$-DP under condition $\Psi$ iff for any neighboring pair of databases $Y,Y'$ both satisfying $\Psi$, we have $D_{e^{\epsilon}}(\cM(Y),\cM(Y')) \leq \delta$. 
\end{fact}

We will need to bound the hockey-stick divergence between two distributions in terms of the hockey-stick divergences to a third distribution. Unfortunately, the hockey-stick divergence does not define a metric and, therefore, does not admit the usual triangle inequality. However, it is possible to prove a looser inequality, which we will find useful:
\begin{lemma} \label{lem:hockeytriangle}
 Suppose $p(x), q(x), r(x)$ are probability density functions on $\R^d$. Then,
 \[
  D_{e^\epsilon}(p,r) \leq D_{e^{\epsilon/2}}(p,q) + e^{\epsilon/2} \cdot  D_{e^{\epsilon/2}}(q,r).
 \]
\end{lemma}

We remark that such a bound is already implicit in the so-called \emph{group differential privacy} (see e.g. \cite[Lemma 2.2]{Vadhan17}). Nonetheless, we provide a (short) proof in \Cref{app:hockeytriangle}.

\subsubsection{Approximate-DP Selection} \label{sec:approx-dp-selection}
Finally, we will also use a DP algorithm for the \emph{selection} problem, where the goal is to pick from a (public) set of candidates one which has a high ``score''. This problem can be solved using the \emph{exponential mechanism}~\cite{McSherryT07}. The version of the algorithm we use deviates slightly from this traditional version in that we also include a check (via truncated Laplace mechanism) to make sure that the score is at least a certain threshold $\kappa$; otherwise, the algorithm's properties are summarized below. Its proof is deferred to \Cref{app:selection}.

\begin{theorem} \label{thm:dp-apx-selection}
Suppose $\eps, \delta \in (0, 1]$. Let $\cC$ be a set of candidates and let $\score: \cC \times \cY$ be a scoring function for candidates as a function of the databases $Y \in \cY$, such that its sensitivity (w.r.t. $Y$) is at most $\Delta$. There exists an algorithm \SelectionAlg that satisfies the following properties:
\begin{enumerate}
\item\label{cond:selectdp} \SelectionAlg is $(\eps, \delta)$-DP.
\item\label{cond:selectgeqkappa} If the output of \SelectionAlg is $c^* \neq \perp$, then $\score(c^*,Y) \geq \kappa$.
\item\label{cond:selectrejectprob} If there exists $c \in \cC$ such that $\score(c, Y) \geq \kappa + O\left(\frac{\Delta}{\eps} \cdot \log\left(\frac{|\cC|}{\beta \delta}\right)\right)$, then \SelectionAlg output $\perp$ with probability at most $\beta$.
\end{enumerate}
\end{theorem}

\section{Differentially Private Robust Moment Estimation}
In this section, we describe a differentially private robust moment estimation algorithm. The following is our main technical result:

\begin{theorem}[Differentially Private Robust Moment Estimation] \label{thm:diff-priv-robust-moment-estimation-section}
Fix $C_0 > 0$ and $k \in \N$. Then, there exists an $\eta_0 > 0$ such that for any given outlier rate $0 < \eta \leq \eta_0$ and $\epsilon,\delta > 0$, there exists a randomized algorithm $\Alg$ that takes an input of $n\geq n_0 = \widetilde{\Omega}\left(\frac{d^{4k}}{\eta^2} \left(1 + \left(\frac{\ln(1/\delta)}{\epsilon}\right)^4 + \left(\frac{\ln(1/\delta)}{\epsilon}\right)^{\frac{2k}{k-1}}\right) \cdot C^{4k} k^{4k+6} \right)$ points $Y \subseteq \bbQ^d$ (where $C = C_0 + \frac{3 \ln(3/\delta)}{\epsilon} + \frac{9}{\epsilon} + 1$), runs in time $(Bn)^{O(k)}$ (where $B$ is the bit complexity of the entries of $Y$) and outputs either ``reject'' or estimates $\hat{\mu} \in \bbQ^d$, $\hat{\Sigma} \in \bbQ^{d\times d}$, and $\hat{M}^{(t)} \in \bbQ^{d\times d \times \cdots \times d}$ (for all even $t < 2k$ such that $t$ divides $2k$) with the following guarantees\footnote{The $\widetilde{\Omega}$ notation hides multiplicative logarithmic factors in $d$, $C$, $k$, $1/\eta$, $1/\epsilon$, and $\ln(1/\delta)$.}:
\begin{enumerate}
    \item \textbf{Privacy: } $\Alg$ is $(\epsilon,\delta)$-differentially private with respect to the input $Y$, viewed as a $d$-dimensional database of $n$ individuals. 
    \item \textbf{Utility: } Suppose there exists a $2k$-certifiably $C_0$-subgaussian set $X \subseteq \bbQ^d$ of $n\geq n_0$ points such that $|Y \cap X| \geq (1-\eta)n$ with mean $\mu_*$, covariance $\Sigma_* \succeq 2^{-\poly(d)}I$, and $t$-th moments $M_*^{(t)}$ for $2\leq t\leq k$. Then, with probability at least $9/10$ over the random choices of the algorithm, $\Alg$ outputs estimates $\hat{\mu} \in \bbQ^d$, $\hat{\Sigma} \in \bbQ^{d\times d}$, and $M^{(t)} \in \bbQ^{d\times d \times \cdots \times d}$ (for all even $t < 2k$ such that $t$ divides $2k$) satisfying the following guarantees:
    \[
    \forall u \in \R^d, \text{  } \iprod{\hat{\mu}-\mu_*,u} \leq O\left(\sqrt{Ck}\right) \eta^{1-1/2k} \sqrt{ u^{\top} \Sigma_*u}\mcom
    \]
    and,
    \[
     \left(1-O((Ck)^{t/2k}) \eta^{1-1/k}\right) \Sigma_* \preceq \hat{\Sigma} \preceq \left(1+O((Ck)^{t/2k})\eta^{1-1/k}\right)\Sigma_* \mcom
    \]
    and, for every even $t < 2k$ such that $t$ divides $2k$, 
    \[
     \left(1-O(Ck) \eta^{1-t/2k}\right) \iprod{u^{\otimes t}, M_*^{(t)}} \leq \iprod{u^{\otimes t}, \hat{M}^{(t)}} \leq \left(1+O(Ck) \eta^{1-t/2k}\right) \iprod{u^{\otimes t}, M_*^{(t)}} \mper
    \]
\end{enumerate}
Moreover, the algorithm succeeds (i.e., does not reject) with probability at least $9/10$ over the random choices of the algorithm.
\end{theorem}

Observe that the privacy guarantees of the algorithm are (necessarily) \emph{worst-case}. The utility guarantees, however, hold only under the assumption that $Y$ is an $\eta$-corruption of a good set $X$.

The above theorem can also be translated into utility guarantees for points sampled from a given distribution by recalling the well-known fact that points sampled from a certifiably subgaussian distribution are good with high probability:
\begin{fact}[See Section 5 in~\cite{KS17}] \label{fact:samplegood} 
Suppose $\cD$ is a certifiably $C$-subgaussian distribution with mean $\mu_*$ and covariance $\Sigma_* \succeq 2^{-\poly(d)}I$ and $t$-moment tensors $M^{(t)}$ for $t \in \N$. For any $k \in \N$, let $X = \{x_1, x_2,\ldots, x_n\}$ be an i.i.d. sample from $\cD$ of size $n \geq n_0 = O(d^{2k}/\eta^2)$. Then, for any $t \in \N$ such that $t$ divides $k$, with probability at least $0.99$ over the draw of $X$, the following all hold:
\begin{enumerate}
\item $X$ is $2k$-certifiably $2C$-subgaussian.
\item $\Norm{\Sigma_*^{-1/2}(\mu(X)-\mu_*)}_2 \leq \eta$.
\item $\Sigma(X) \in (1 \pm \eta) \Sigma_*$.
\item $\sststile{2k}{v} \Set{ \iprod{v^{\otimes t}, M^{(t)}(X)} \in (1 \pm \eta) \iprod{v^{\otimes t}, M^{(t)}_*} }$.
\end{enumerate}
\end{fact}

We note that our main theorem for private robust moment estimation, \cref{thm:diff-priv-robust-moment-samples}, is an immediate consequence of \cref{thm:diff-priv-robust-moment-estimation-section} and \cref{fact:samplegood}.

For the rest of the section, we will work to prove \cref{thm:diff-priv-robust-moment-estimation-section}. In \cref{subsec:witnessalg}, we will introduce a witness-producing robust moment estimation algorithm that will be used as a subroutine for our main algorithm and present relevant utility guarantees. In \cref{subsec:privaterobustalg}, we will then introduce our main algorithm. After that, we will prove the necessary privacy guarantees in \cref{subsec:privacyanalysis}. Finally, we will put together the pieces to prove our main theorem, \cref{thm:diff-priv-robust-moment-estimation-section}, in \cref{subsec:mainproof}.

\subsection{Witness-Producing Version of Robust Moment Estimation Algorithm} \label{subsec:witnessalg}
As a key building block, we will use the following (non-private) version of the robust moment estimation algorithm of~\cite{KS17} that uses the same constraint system $\cA$ as in~\cite{KS17}. Our algorithm itself, however, makes one key change (we call our version ``witness-producing'' for reasons that will soon become clear) to that of~\cite{KS17} in order to obtain a private robust moment estimation algorithm. Instead of outputting estimates of the moments of the unknown distribution, our algorithm  outputs a sequence of non-negative weights $p_1, p_2,\ldots, p_n$ forming a probability distribution on the input set of points $Y$. The estimates can then be obtained by taking moments of the finite set $Y$ with respect to the probability distribution on $Y$ defined by the weights $p_i$s. This simple change is crucial to our \emph{worst-case} analysis of the resulting algorithm (i.e. even when the distributional assumption that $Y$ is an $\eta$-corruption of some good set $X$ is not met) and obtaining our privacy guarantees. As we discuss, our blueprint for modifying convex optimization based robust estimation algorithms appears to broadly applicable beyond the specific setting of robust moment estimation.

The underlying constraint system $\cA$ is shown below, and the witness-producing robust moment estimation algorithm is shown as Algorithm~\ref{algo:robust-moment-estimation-witness-producing}.

\begin{mdframed}[frametitle={$\cA_{C,k,\eta,n}(\{y_1,y_2,\dots,y_n\})$: Constraint System for $\eta$-Robust Moment Estimation}]
        \begin{enumerate}
            \item $w_i^2 = w_i$ for each $1 \leq i \leq n$,
            \item $\sum_{i = 1}^n w_i \geq (1-\eta)n$,
            \item $\mu' = \frac{1}{n} \sum_i x_i'$,  
            \item $w_i (x_i' - y_i) = 0$ for $1 \leq i \leq n$, 
            \item $\frac{1}{n} \sum_{i = 1}^n \iprod{x_i'-\mu',v}^k \leq (Ck)^{k/2} \Paren{\frac{1}{n} \sum_{i = 1}^n \iprod{x_i'-\mu',v}^2}^{k/2}$. 
        \end{enumerate}
        \label{box:constraints}
\end{mdframed}

\begin{mdframed}
      \begin{algorithm}[Witness-Producing Robust Moment Estimation]
        \label{algo:robust-moment-estimation-witness-producing}\mbox{}
        \begin{description}
        \item[Given:]
        A set of points $Y = \{y_1, y_2, \ldots, y_n\} \subseteq \bbQ^d$, $\eta > 0$, a parameter $k \in \N$.
        \item[Output:]
        Either ``reject'' or non-negative weights $p_1, p_2, \ldots, p_n$ s.t. $p_i \leq \frac{1}{(1-\eta)n} \text{  } \forall i$ and $\sum_i p_i = 1$. 
        \item[Operation:]
         \leavevmode
         \begin{enumerate}
         \item Find a pseudo-distribution $\tzeta$ of degree $O(k)$ (\avnote{Is this $2k$?}) satisfying the constraint system $\cA_{C,k,\eta,n}(Y)$. If such a pseudo-distribution does not exist, then return ``reject.'' 
         \item Output weights $p \in [0,1]^n$ defined by $p_i = \frac{\pE_{\tzeta}[w_i]}{\sum_{i = 1}^n \pE_{\tzeta}[w_i]}$ for each $i$.  
            \end{enumerate} 
        \end{description}
      \end{algorithm}
\end{mdframed}

\paragraph{Analysis of the witness-producing robust estimation algorithm}
Robust estimation algorithms that rely on the use of semidefinite programming are all analyzed under distributional assumptions on the input set of points. Roughly speaking, such algorithms search over set of points that have a large enough intersection with the input corrupted sample and satisfy certain relevant property of the underlying family of distributions. In order to obtain privacy guarantee that holds for worst-case inputs, we need to upgrade the analyses of such algorithms so that they not only provide estimates of the target parameters, but also explicitly produce ``witnesses''---these are subsets of the input corrupted sample that define distributions with the estimated parameters and further, satisfy the relevant property of the underlying family of distributions. 

In this section, we verify that such a stronger guarantee can be obtained for robust moment estimation algorithm of~\cite{KS17}. Formally, their algorithm succeeds as long as the input is an $\eta$-corruption of a certifiably subgaussian set.

The following guarantees for the algorithm above were shown in~\cite{KS17}. 
\begin{fact}[Lemmas 4.4, 4.5, and 4.8 in~\cite{KS17}] \label{fact:ks-analysis}
Let $X \subseteq \R^d$ be a set of size $n$ that is $2k$-certifiably $C$-subgaussian with mean $\mu_*$, covariance $\Sigma_*$ and $t$-th moment $M^{(t)}_*$ for $t$ evenly dividing $2k$. Let $Y$ be an $\eta$-corruption of $X$. Then, for $\mu' = \frac{1}{n} \sum_i x_i'$, $\Sigma' = \frac{1}{n} \sum_i (x_i-\mu')(x_i-\mu')^{\top}$, and ${M^{(t)}}'= \frac{1}{n} \sum_i {x_i'}^{\otimes t}$, we have:
\[
\cA \sststile{2k}{u} \Set{ \iprod{\mu'-\mu_*,u}^{2k} \leq O(C^{k}k^k) u^{\top} \Sigma_* u^{k}}\mcom
\]
\[
\cA \sststile{2k}{u} \Set{ \iprod{\Sigma'-\Sigma_*,u^{\otimes 2}}^{k} \leq O(C^{k}k^{k}) u^{\top} \Sigma_* u^{k}}\mcom
\]
\[
\cA \sststile{2k}{u} \Set{ \iprod{{M^{(t)}}'-M^{(t)}_*,u^{\otimes t}}^{2k/t} \leq O(C^{k}k^{k}) u^{\top} \Sigma_* u^{k}}\mper
\]
\end{fact}

\begin{lemma}[Guarantees for Witness-Producing Robust Moment Estimation Algorithm] \label{lem:witness-producing-robust-moment-estimation}
Given a subset of of $n$ points $Y \subseteq \bbQ^d$ whose entries have bit complexity $B$, Algorithm~\ref{algo:robust-moment-estimation-witness-producing} runs in time $(Bn)^{O(k)}$ and either (a.) outputs ``reject,'' or (b.) returns a sequence of weights $0 \leq p_1, p_2, \ldots, p_n$ satisfying $p_1 + p_2 + \cdots + p_n = 1$.

Moreover, if $X \subseteq \R^d$ is $2k$-certifiably $C$-subgaussian with mean $\mu_*$, covariance $\Sigma_*$ and in general, $t$-th moment tensor $M^{(t)_*}$ such that $|Y \cap X| \geq (1-\eta)n$, then \cref{algo:robust-moment-estimation-witness-producing} never rejects, and the corresponding estimates $\hat{\mu} = \frac{1}{n} \sum_i p_i y_i$ and $\hat{\Sigma} = \sum_{i = 1}^n p_i (y_i -\hat{\mu})(y_i - \hat{\mu})^{\top}$ satisfy the following guarantees for $\beta_t = O(C^{t/2}k^{t/2})\eta^{1-t/2k}$ for $t \leq k$:
\begin{enumerate}
   \item \textbf{Mean Estimation: } \[
    \forall u \in \R^d, \text{  } \iprod{\hat{\mu}-\mu_*,u} \leq O(\sqrt{Ck}) \eta^{1-1/2k} \sqrt{ u^{\top} \Sigma_*u}\mcom
    \]
    \item \textbf{Covariance Estimation: }\[
     (1-\beta_2) \Sigma_* \preceq \hat{\Sigma} \preceq (1+\beta_2)\Sigma_* \mcom
    \]
    \item \textbf{Moment Estimation: } For all even $t < 2k$ such that $t$ divides $2k$,
    \[
    \forall u \in \R^d, \text{  } (1-\beta_t) \iprod {u^{\otimes t}, M^{(t)}_*} \leq  \iprod{u^{\otimes t}, \hat{M}^{(t)}} \leq (1+\beta_t) \iprod {u^{\otimes t}, M^{(t)}_*}
    \]
    \item \textbf{Witness: } For $C' \leq C(1 + O(\eta^{1-1/k}))$, 
    \[
   \sststile{}{} \Set{ \frac{1}{n} \sum_{i=1}^n p_i \iprod{y_i - \hat{\mu}}^{2k} \leq  (C'k)^k \Paren{\frac{1}{n} \sum_{i=1}^n p_i \iprod{y_i - \hat{\mu}}^{2}}^k }
    \]
\end{enumerate}
\end{lemma}

The first three properties follow easily from an analysis similar to the one in~\cite{KS17}. We verify the last property  below. 
\begin{lemma} \label{lem:witness-property-analysis}
Let $\tzeta$ be a pseudo-distribution of degree $O(k)$ consistent with $\cA$ on input $Y$ with outlier rate $\eta \ll 1/k$. 
Suppose there exists a $2k$-certifiably $C_1$-subgaussian distribution $X \subseteq \R^d$ with mean $\mu_*$ of size $n$ such that $|Y \cap X|\geq (1-\eta)n$. Then, for $\eta \leq \eta_0$ for some absolute constant $\eta_0$ and for $\hat{\mu} =\frac{1}{W} \sum_{i = 1}^n  \pE_{\tzeta}[w_i] y_i$ where $W = \sum_{i =1}^n \pE[w_i]$, we have: 
\[
\sststile{2k}{u} \Set{\frac{1}{W} \sum_{i = 1}^n \pE_{\tzeta}[w_i] \iprod{y_i-\hat{\mu},u}^{2k} \leq (C' k)^k\Paren{\frac{1}{W} \sum_{i = 1}^n \pE_{\tzeta}[w_i] \iprod{y_i-\hat{\mu},u}^{2}}^k}\mcom
\]
for $C' \leq C (1+ O(\eta^{1-1/2k})k) \leq C+1$ for small enough $\eta$. 

\end{lemma}
\begin{proof}
We have:
\begin{equation*}
\Paren{\frac{1}{n} \sum_{i = 1}^n \pE_{\tzeta}[w_i] \iprod{y_i - \hat{\mu},u}}^{2k} = \frac{1}{n} \sum_{i = 1}^n \pE_{\tzeta}[w_i\iprod{x_i'-\hat{\mu},u}^{2k}] \leq \frac{1}{n} \sum_{i = 1}^n \pE_{\tzeta}[ \iprod{x_i'-\mu'+ \mu'-\hat{\mu},u}^{2k}]
\end{equation*}
The first term on the right-hand side above is at most $(Ck)^k \pE_{\tzeta}[(\frac{1}{n} \sum_{i = 1}^n \iprod{x_i'-\mu',u}^{2})^k] \leq (C(1+ O(\eta^{1-1/2k})k)^k u^{\top}\Sigma_* u^k$ using certifiable subgaussianity constraints and Fact~\ref{fact:ks-analysis}. 

Let us analyze the 2nd term above. 

\begin{align}
&\frac{1}{n} \sum_{i = 1}^n \pE_{\tzeta}[ \iprod{x_i'-\mu'+ \mu'-\hat{\mu},u}^{2k}] \\
&= \frac{1}{n} \sum_{i = 1}^n \pE_{\tzeta}[ \iprod{x_i'-\mu',u}^{2k}] + 2k \frac{1}{n} \sum_{i = 1}^n \pE_{\tzeta}[ \iprod{x_i'-\mu',u}^{2k-2} \iprod{\mu'-\hat{\mu}, u}^2 \\
&+ \sum_{j = 2}^{2k} {{2k} \choose j} \frac{1}{n} \sum_{i = 1}^n \pE_{\tzeta}[ \iprod{x_i'-\mu',u}^{2k-j} \iprod{\mu'-\hat{\mu}, u}^j]\\
&\leq \frac{1}{n} \sum_{i = 1}^n \pE_{\tzeta}[ \iprod{x_i'-\mu',u}^{2k}] + 2k \frac{1}{n} \sum_{i = 1}^n (\pE_{\tzeta}[ \iprod{x_i'-\mu',u}^{2k})^{(2k-2)/2k} (\pE_{\tzeta}\iprod{\mu'-\hat{\mu}, u}^{2k})^{1/2k} \\
&+ \sum_{j = 2}^{2k} {{2k} \choose j} \frac{1}{n} \sum_{i = 1}^n \pE_{\tzeta}[ \iprod{x_i'-\mu',u}^{2k-j} \iprod{\mu'-\hat{\mu}, u}^j] \label{eq:first-bound-transference}
\end{align}
Here, in the 2nd inequality, we used the Hölder's inequality for pseudo-distributions. Let us analyze the 2nd term in the right-hand side above by observing the following that uses the bounds from Fact~\ref{fact:ks-analysis}:
\begin{align}
&\pE_{\tzeta}[\iprod{\mu'-\hat{\mu},u}^{2k}] \leq 2^{2k} [\pE_{\tzeta}[\iprod{\mu'-\mu_*,u}^{2k}] + 2^{2k} \iprod{\mu_*-\hat{\mu}, u}^{2k}\\
&\leq 2^{2k} (Ck)^k \eta^{2k-1}  u^{\top} \Sigma_* u^k + 2^{2k} (Ck)^k \eta^{2k-1} (1+\beta_2)^k u^{\top} \Sigma_* u^k
\end{align}

This allows us to infer that the 2nd term in \eqref{eq:first-bound-transference} is at most $\pE_{\tzeta}[\frac{1}{n} \sum_{i = 1}^n \iprod{x_i'-\mu',u}^{2k}]^{(2k-2)/2k} \cdot (5Ck)^{1/2} \eta^{1-1/2k} \sqrt{u^{\top} \Sigma_* u} \leq  O(k) (Ck)^k \eta^{1-1/2k} u^{\top} \Sigma_* u^k$ using certifiable subgaussianity constraints and Fact~\ref{fact:ks-analysis}. 

Let's now analyze the terms corresponding to $j\geq 2$ in the right-hand side of \eqref{eq:first-bound-transference}. Each of these terms corresponds to a ``mixed monomial'' in $\iprod{x_i'-\mu',u}$ and $\iprod{\mu'-\mu,u}$. Let us first analyze the even individual degree terms. 

First observe that by Hölder's inequality for pseudo-distributions again, we have:
\begin{equation} \label{eq:basic-even-bound}
\frac{1}{n} \sum_{i = 1}^n \pE_{\tzeta}[ \iprod{x_i'-\mu',u}^{2k-2} \iprod{\mu'-\hat{\mu}, u}^{2}] \leq  \iprod{x_i'-\mu',u}^{2k})^{(k-1)/k} (\pE_{\tzeta}\iprod{\mu'-\hat{\mu}, u}^{2k})^{1/k}\mper
\end{equation}
By an analysis similar to the case of the first term on the right-hand side of \eqref{eq:first-bound-transference} above, we obtain that the right-hand side is at most: $O(1) (Ck)^k (\eta^{1-1/2k})^2 u^{\top} \Sigma_* u^k$.

Next, let's analyze all terms corresponding to even $j$. By Proposition~\ref{prop:mixed-vs-pure-monomials}, we have:

\begin{align*}
\frac{1}{n} \sum_{i = 1}^n \pE_{\tzeta}[ \iprod{x_i'-\mu',u}^{2k-2j} \iprod{\mu'-\hat{\mu}, u}^{2j}] &\leq  \frac{1}{n} \sum_{i = 1}^n (\pE_{\tzeta}[ \iprod{\mu'-\hat{\mu}, u}^{2} \iprod{x_i'-\mu',u}^{2k-2j} \iprod{\mu'-\hat{\mu}, u}^{2j-2}]\\
&\leq \frac{2k}{n} \sum_{i = 1}^n (\pE_{\tzeta}[ \iprod{\mu'-\hat{\mu}, u}^{2} (\iprod{x_i'-\mu',u}^{2k-2} + \iprod{\mu'-\hat{\mu}, u}^{2k-2})]
\end{align*}
The first term can now be upper bounded by the bound for \eqref{eq:basic-even-bound} and the 2nd term by an application of Fact~\ref{fact:ks-analysis}. 

The case of odd terms is similar with the first step using Proposition~\ref{prop:even-vs-odd-monomials}. 

Altogether, we obtain an upper bound of $(C(1+ O(\eta^{1-1/2k})k)^k u^{\top}\Sigma_* u^k$. 




On the other hand, using the sum-of-squares version of the Cauchy-Schwarz inequality along with the almost triangle inequality and invoking Fact~\ref{fact:ks-analysis} we have:
\begin{align*}
\cA &\sststile{2k}{u} \Biggl\{ \Paren{\frac{1}{n} \sum_{i = 1}^n (1-w_i) \iprod{x_i'-\hat{\mu},u}^2}^2 \leq \Paren{\frac{1}{n} \sum_{i = 1}^n (1-w_i)^2} \frac{1}{n} \sum_{i = 1}^n \iprod{x_i'-\hat{\mu},u}^4\\
 &\leq 16\eta C^2 \Paren{\frac{1}{n} \sum_{i = 1}^n \iprod{x_i'-\mu',u}^2 + \frac{1}{n} \sum_{i = 1}^n \iprod{\mu'-\hat{\mu},u}^2}^2 \leq 20 \eta C^2 (1+\beta_2)^2 u^{\top} \Sigma_* u^2 \Biggr\}
\end{align*}

Thus, 
\begin{align*}
\cA &\sststile{2k}{u} \Biggl\{ \Paren{\frac{1}{n} \sum_{i = 1}^n \pE_{\tzeta}[w_i \iprod{y_i-\hat{\mu},u}^2}^2= \Paren{\frac{1}{n} \sum_{i = 1}^n \pE_{\tzeta}[\iprod{y_i-\hat{\mu},u}^2}^2] - \Paren{\frac{1}{n} \sum_{i = 1}^n \pE_{\tzeta}[(1-w_i) \iprod{y_i-\hat{\mu},u}^2}^2 n\\
&\geq (1-O(Ck)\eta^{1-1/k}-80 \eta C^2 ) u^{\top} \Sigma_* u  \mper
\Biggr\}
\end{align*}
The lemma now follows immediately for small enough fixed constant $\eta$.

\end{proof}

\subsection{Private Robust Moment Estimation} \label{subsec:privaterobustalg}
We are now ready to present our main algorithm for private robust moment estimation. Our algorithm uses the witness-producing algorithm (Algorithm~\ref{algo:robust-moment-estimation-witness-producing}) as a major building block while augmenting it to search for pseudo-distributions that, in addition to satisfying the relevant set of constraints, also minimize an appropriate strongly convex potential function. We define the relevant potential function $\Pot$ below in \cref{def:potentialfunc}.

\begin{definition}[Potential Function] \label{def:potentialfunc}
Let $C > 0$ and $n, k\in\N$.
For any pseudo-distribution $\tzeta$ of degree $2$ consistent with $\cA_{C,k,\eta,n}(Y)$ for outlier rate $\eta$ and input $Y\subseteq \R^d$, let $\Pot_{\eta,\tzeta}^{C,k,n}(Y)$ be defined as $\Norm{\pE_{\tzeta}[w]}_2^2$. Furthermore, let $\Pot_{\eta}^{C,k,n}(Y) = \min_{\tzeta \text{ sat } \cA_{C,k,\eta,n}(Y)} \Pot_{\eta,\tzeta}^{C,k,n}(Y)$ be the minimum value of the potential as $\tzeta$ ranges over all pseudo-distributions of degree $2t$ consistent with $\cA_{C,k,\eta,n}(Y)$.  If no such pseudo-distribution exists, set $\Pot_{\eta}(Y) = \infty$.

When $C,n,k$ are understood from context, we may suppress these parameters and simply write $\Pot_\eta$ and $\Pot_{\eta,\tzeta}$.
\end{definition}

Now, we are ready to describe our main private robust moment algorithm, which is listed as Algorithm~\ref{algo:private-robust-moment-estimation}. The algorithm consists of three main steps. In the first step, the randomized DP selection algorithm (\Cref{thm:dp-apx-selection}) is used to pick an outlier rate (according to a suitable scoring function, as defined below in \Cref{def:em-score}). The second step invokes the witness-producing algorithm (Algorithm~\ref{algo:robust-moment-estimation-witness-producing}) with the outlier rate chosen in step 1, after which one checks that the outputted weights induce a certifiably subgaussian distribution on the input dataset $Y$. Finally, in the last step, one takes the estimates of the mean, covariance, and higher moments provided by the resulting weight vector and adds suitable noise to guarantee differential privacy.

\begin{mdframed}
      \begin{algorithm}[Private Robust Moment Estimation]
        \label{algo:private-robust-moment-estimation}\mbox{}
        \begin{description}
        \item[Given:]
        A set of points $Y = \{y_1, y_2, \ldots, y_n\} \subseteq \bbQ^d$, parameters $C, \eta, \epsilon,\delta > 0$, $L, k \in \N$.
        \item[Output:]
        Estimates $\hat{\mu}$, $\hat{\Sigma}$, and $\hat{M}^{(t)}$ ($3\leq t\leq k$) for mean, covariance, and $t$-moments.
        \item[Operation:]\mbox{}
        \begin{enumerate}
         \item  \textbf{Stable Outlier Rate Selection: } Use the $(\eps/3, \delta/3)$-DP \SelectionAlg with $\kappa = L/2$ to sample an integer $\tau \in [\eta n]$ with the scoring function as defined in~\Cref{def:em-score}. If $\tau = \perp$, then reject and halt. Otherwise, let $\eta' = \tau / n$. \label{step:outlier-rate-selection}
      	  \item \textbf{Witness Checking: }Compute a pseudo-distribution $\tzeta$ of degree $2k$ satisfying $\cA_{C,k,\eta',n}(Y)$ and minimizing $\Pot_{\eta',\tzeta}(Y)$. Let $\gamma \sim \tLap\left(-\left(1 + \frac{3\ln\left(3/\delta\right)}{\eps}\right),3/\eps\right)$ and $C' = C+\gamma$. Check that the weight vector $p= \pE_{\tzeta}[w]$ induces a $C'$-certifiably subgaussian distribution on $Y$. If not, reject immediately. Otherwise, let $\widetilde{\mu} = \pE_{\tzeta}[\mu]$, $\widetilde{\Sigma} = \pE_{\tzeta}[\Sigma]$, and $\widetilde{M}^{(t)} = \pE_{\tzeta}[M^{(t)}]$ (for all even $t < 2k$ such that $t$ divides $2k$) be the mean, covariance, and $t^\text{th}$ moment estimates, respectively, that are induced by the pseudo-distribution $\tzeta$.  \label{step:witness-checking} 
          \item \textbf{Noise Addition: } \label{step:noise-addition} Let $\gamma_1 = O(C'k) (L/n)^{\frac{1}{2}\left(1-\frac{1}{2k}\right)}$ and $\gamma_t = O((C'k)^{t/2}) (L/n)^{\frac{1}{2}\left(1-\frac{t}{2k}\right)}$ for $t \geq 2$. Let $z \sim \cN(0,\sigma_1)^d$ and $Z \sim \cN(0,\sigma_2)^{{d+1}\choose 2}$, where we interpret $Z$ as a symmetric $d \times d$ matrix with i.i.d. entries in the upper triangular portion. Similarly, for $t \geq 2$, let $Z^{(t)} \sim \cN(0,\sigma_t)^{{d+(t-1)} \choose t}$, where we interpret $Z$ as a symmetric $\underbrace{d\times d \times \cdots d}_{\text{$t$ times}}$ tensor with ${d+(t-1)} \choose t$ independent ``upper-triangular'' entries. Moreover, let 
          \[
          \begin{cases}
          \sigma_j = 6k\epsilon^{-1}\gamma_j d^{\frac{t-1}{2}} \sqrt{2\ln(7.5k/\delta)}, \quad\text{for $j=1,2$} \\
          \sigma_j = 6k\epsilon^{-1}\gamma_j (C'k)^t d^{\frac{t-1}{2}} \sqrt{2\ln(7.5k/\delta)}, \quad\text{for $j > 2$}
          \end{cases}\mper
          \]
          Then, output:
          \begin{itemize}
           \item $\hat{\mu} = \widetilde{\mu} + \widetilde{\Sigma}^{1/2} z$.
           \item $\hat{\Sigma} = \widetilde{\Sigma} + \widetilde{\Sigma}^{1/2}Z \widetilde{\Sigma}^{1/2}$.
           \item $\hat{M}^{(t)} = \widetilde{M}^{(t)} + ((\widetilde{\Sigma} + \widetilde{\mu}\widetilde{\mu}^T)^{1/2})^{\otimes t} Z^{(t)}$, for all even $t < 2k$ such that $t$ divides $2k$.
          \end{itemize}
        \end{enumerate}

        \end{description}
      \end{algorithm}
\end{mdframed}

\subsection{Privacy Analysis} \label{subsec:privacyanalysis}
Our analysis of the privacy of \Alg is based on a sequence of claims about each of the steps of \Alg that cumulatively establish the stability of the behavior of \Alg on adjacent inputs $Y,Y'$. We will rely on the following simple but key observation in our analysis. 
It is easy to verify using the definition of pseudo-distributions. 
\begin{lemma}[Adjacent Pseudo-distributions] \label{lem:adjacent-pseudo-distribution}
Let $\tzeta$ be a pseudo-distribution of degree $2k$ that satisfies all the constraints in \ref{box:constraints} on input $Y = \{y_1, y_2, \ldots, y_n\}$ with outlier rate $\eta$. Let $Y' \subseteq \R^d$ be adjacent to $Y$. Define an \emph{adjacent pseudo-distribution} $\tzeta'$ (that ``zeroes out $w_i$'') by $\pE_{\tzeta'}[ w_S p(X',\cdots)] = \pE_{\tzeta}[w_S p(X',\cdots)]$ if $i \not \in S$ and $\pE_{\tzeta'}[ w_S p(X',\cdots)] = 0$ if $i \in S$ for every polynomial $p$ in $X'$ and other auxiliary indeterminates in $\cA$. Then, $\tzeta'$ is a pseudo-distribution of degree $2k$ that satisfies all the constraints in \ref{box:constraints} on both inputs $Y'$ and $Y$ with outlier parameter $\eta+1/n$. 
\end{lemma}

This allows us to conclude the following basic calculus of our potential function:
\begin{lemma}[Basic Facts about $\Pot$] \label{lem:basic-props-of-pot}
Suppose that for some $Y \subseteq \R^d$ of size $n$, some $t \in \N$ and $\eta' \in [0,\eta_0/4]$, there is a pseudo-distribution of degree $2t$ consistent with $\cA$ on input $Y$. Then, for every $\eta\geq \eta'$, the following holds:
\begin{enumerate}
    \item \textbf{Monotonicity:} $\Pot_{\eta+1/n}(Y) \leq \Pot_{\eta}(Y)$. In particular, $\Pot$ is monotonically decreasing as its subscript increases.
    \item \textbf{Lower Bound: } $\Pot_{\eta}(Y) \geq (1-\eta)^2 n$. 
    \item \textbf{Upper Bound: } $\Pot_{\eta}(Y) \leq (1-\eta)n$.
\end{enumerate}
\end{lemma}
\begin{proof}
The first fact follows immediately from Lemma~\ref{lem:adjacent-pseudo-distribution}. For the second, observe that any pseudo-distribution $\tzeta$ of degree $2t$ consistent with $\cA$ on input $Y$ with outlier rate $\eta$ must satisfy $\sum_{i = 1}^n \pE_{\tzeta}[w_i] \geq (1-\eta)n$. Thus, by Cauchy-Schwarz inequality, $\sum_{i=1}^n \pE[w_i]^2 \geq (\sum_{i=1}^n \pE[w_i])^2/n = (1-\eta)^2 n$. This completes the proof. For the last part, observe that $\pE_{\tzeta}[w_i] \leq 1$ for every $i$. Thus, $\sum_{i = 1}^n \pE_{\tzeta}[w_i]^2 \leq \sum_{i = 1}^n \pE_{\tzeta}[w_i] = (1-\eta)n$. 
\end{proof}

\paragraph{Analysis of stable outlier rate selection} 
The goal of the first step of \Alg is to find an outlier rate $\eta'$ such that the strongly convex potential function $\Pot(\tzeta)$ on the pseudo-distribution we will eventually compute (in Step 3) is close on adjacent input points $Y,Y'$. We will later use the strong convexity of the $\Pot$ and the closeness guarantee on $\Pot$ on $Y,Y'$ to infer that the weight vector $p(Y)$ and $p(Y')$ output by the algorithm themselves are close. 

Our key algorithmic trick to ensure the closeness of the strongly convex potential $\Pot$ is to find a ``stable interval'' $[\eta' - 0.5L/n, \eta' + 0.5L/n]$ of outlier rates $\eta''$ such that strongly convex potential function at near-optimal solutions must vary slowly as $\eta''$ varies in the the interval. We find such an interval via a variant of the exponential mechanism. 



\begin{definition}[Stability] \label{def:stability}
Fix $L \in \N$. Let $\tau, \gamma \in \{0, \dots, n\}$ such that $\gamma \leq \tau, n - \tau$. Suppose for some $Y \subseteq \R^d$ of size $n$, the constraint system $\cA((\tau-\gamma)/n)$ is feasible. We define the stability of the $2\gamma$ length interval centered at $\tau$ to be 
\begin{align*}
\stab_Y(\tau,\gamma) = \Pot_{(\tau-\gamma)/n}(Y)- \Pot_{(\tau+\gamma)/n}(Y)
\end{align*}
\end{definition}

Observe that if there is a pseudo-distribution consistent with $\cA$ on $Y$ with outlier rate $(\tau-\gamma)/n$ then there is a pseudo-distribution consistent with $\cA$ on $Y$ with any outlier rate $\geq (\tau-\gamma)/n$. Thus, stability above is well-defined. 

\begin{definition}[Score Function] \label{def:em-score}
Fix $n,k\in\N$ and $C > 0$. Let $Y \subseteq \R^d$ be a set of size $n$. For a parameter $L$, we define the following score function for every integer $\tau\in [n]$:
\begin{align*}
\score_{n,C,k}(\tau, Y) = \begin{cases}
0 & \text{ if } \Alg(Y,\tau/n) \text{ is infeasible}, \\
\max_{\gamma \atop \cA_{C,k,(\tau-\gamma)/n,n}(Y) \text{ is feasible}} \min\{\gamma, 20L - \stab_Y(\tau, \gamma)\} &\text{ otherwise}.
\end{cases}
\end{align*}
In the second case, we define $\gamma^*_Y(\tau) := \arg\max_{\gamma \atop \cA_{C,k,(\tau-\gamma)/n,n}(Y) \text{ is feasible}} \min\{\gamma, L - \stab_Y(\tau, \gamma)\}$.
\end{definition}

\begin{lemma} \label{lem:comparison-of-stability}
Let $\tau, \gamma \in [n]$ such that $\gamma \leq \tau, n - \tau$. Suppose for some $Y \subseteq \R^d$ of size $n$, the constraint system $\cA((\tau-\gamma)/n)$ is feasible for both $Y$. Let $Y'$ be any collection of $n$ points in $\R^d$ differing from $Y$ in at most one point. Then, for any $\tau, \gamma$, 
\[
\stab_{Y'}(\tau,\gamma-1) \leq \stab_{Y}(\tau,\gamma)
\]
\end{lemma}

\begin{proof}
Using Lemma~\ref{lem:adjacent-pseudo-distribution} and noting that if $\tzeta'$ is adjacent to $\tzeta$ then $\Norm{\pE_{\tzeta'}[w]}_2^2 \leq \Norm{\pE_{\tzeta}[w]}_2^2$, we have:
\[
 \Pot_{(\tau-\gamma+1)/n}(Y') \leq \Pot_{(\tau-\gamma)/n}(Y),
\]
and
\[
 \Pot_{(\tau+\gamma)/n}(Y) \leq \Pot_{(\tau+\gamma-1)/n}(Y').
\]
Combining the two equations yields
\begin{align*}
 \stab_{Y'}(\tau,\gamma-1) &= \Pot_{(\tau-\gamma+1)/n}(Y')-\Pot_{(\tau+\gamma-1)/n}(Y')\\
 &\leq \Pot_{(\tau-\gamma)/n}(Y)-\Pot_{(\tau+\gamma)/n}(Y)
 = \stab_{Y}(\tau,\gamma). \qedhere
\end{align*}
\end{proof}

\begin{lemma}[Sensitivity of Score Function] \label{lem:sensitivity-score-function}
Let $Y,Y'$ be set of $n$ points in $\R^d$ differing at most in one point, and $\tau \in [n]$. Then, for every $\tau >0$,
\begin{equation}
 |\score(\tau, Y)-\score(\tau, Y')| \leq 2. \label{eq:scoresensitivity}
\end{equation}
\end{lemma}

\begin{proof}
It suffices to prove that $\score(\tau, Y') \geq \score(\tau, Y)-2$. A symmetric argument then proves that $\score(\tau, Y) \geq \score(\tau, Y')-2$, which establishes \eqref{eq:scoresensitivity}.

Consider the following two cases:
\begin{itemize}
\item $\Alg(Y,(\tau - 1)/n)$ is infeasible for $Y$ or $Y'$. In this case, we have $\score(\tau, Y) \leq 2$, which implies the desired bound.
\item $\Alg(Y,(\tau - 1)/n)$ is feasible for both $Y$ and $Y'$. 
\end{itemize}

Let $\gamma^* := \gamma^*_Y(\tau)$.
From Lemma~\ref{lem:comparison-of-stability}, we know that $\stab(\tau,\gamma^*-1, Y') \leq \stab_Y(\tau,\gamma^*) + 2$. Thus, it follows that
\begin{align*}
 \score(\tau, Y') &\geq \min\{\gamma^* - 1, 20L - \stab_{Y'}(\tau, \gamma^* - 1) \}\\
 &\geq \min\{\gamma^* - 1, 20L - \stab_{Y}(\tau, \gamma^*)\}
 &\geq \score_Y(\tau) - 1,
\end{align*}
as desired.
\end{proof}

\begin{lemma}[Existence of a Good Stable Interval] \label{lem:existence-of-good-interval}
Suppose $\cA(\eta/2)$ is feasible on $Y$. For every $L \in [0, 0.25 \eta n]$, there is a $\tau \in [0,\eta n]$ such that $\score(\tau, Y) \geq L$. 
\end{lemma}

\begin{proof}
Consider $\Pot_{\eta/2}, \Pot_{\eta/2 + 2L/n}, \dots, \Pot_{\eta/2 + 2Lr/n}$ where $r := \lfloor 0.25\eta n / L \rfloor$. Observe that $\Pot_{\eta/2}(Y)-\Pot_{\eta}(Y) \leq (1-\eta/2)n - (1-\eta)^2n \leq 1.5 \eta n$. Therefore, there must exists $r^* \in [r]$ such that
$$\Pot_{\eta/2 + 2L(r^* - 1)/n} - \Pot_{\eta/2 + 2Lr^*/n} \leq \frac{1.5\eta n}{r} \leq 12L.$$
Let $\tau = \eta/2 + (2Lr^* - 1)/n$ and $\gamma = L$. Then, we have $\stab(\tau, \gamma) \leq 12L$ and, thus,
\begin{align*}
\score(\tau, Y) \geq \max\{\gamma, 20L - 12L\} \geq L.
\end{align*}
\end{proof}

\begin{lemma}[Utility of Score Function] \label{lem:util-score}
Suppose $\cA(\eta/2)$ is feasible on $Y$. Let $\epsilon, \delta, \beta \in (0, 1]$. For every $L \in [0, 0.25 \eta n]$, if $L \geq O\left(\frac{1}{\eps} \cdot \log\left(\frac{n}{\beta \delta}\right)\right)$, then with probability $1 - \beta$, \Cref{thm:dp-apx-selection}, invoked with the score function in \cref{def:em-score} and $\kappa = L/2$, does not reject, and the output $\tau$ satisfies $\stab_{Y}(\tau, L/2) < 20 L$.
\end{lemma}

\begin{proof}
This follows from the guarantee of \SelectionAlg (\Cref{thm:dp-apx-selection}), \Cref{lem:existence-of-good-interval} and the definition of $\score$.
\end{proof}

\begin{lemma}[Potential Stability Under Good Coupling] \label{lem:pot-stability}
Let $\eta, \epsilon, \delta > 0$ and $k,L\in \N$ be given input parameters such that $0.25\eta n \geq L = \Omega\left(\frac{1}{\epsilon}\cdot \log\left(\frac{n}{\beta\delta}\right)\right)$. Let $Y,Y'$ be adjacent subsets of $\bbQ^d$. Suppose \Alg does not halt and chooses $\eta' = \tau / n$ in Step~\ref{step:outlier-rate-selection} on input $Y$ and $Y'$. Then, 
\[
\Abs{\Pot_{\eta'}(Y) - \Pot_{\eta'}(Y')} \leq 20L \mper
\]
Consequently, if $p,p'$ are scalings of $\pE_{\tzeta}[w]$ and $\pE_{\tzeta'}[w]$ so that $\norm{p}_1 = \norm{p'}_1 = 1$, then, 
\[
\Norm{p-p'}_1 \leq 120 \sqrt{L/n} \mper
\] 
\end{lemma}
\begin{proof}
It is enough to prove that $\Pot_{\eta'}(Y) - \Pot_{\eta'}(Y') \leq 20L$ as a symmetric argument proves the other direction and completes the proof. 

Let $\tzeta$ be the pseudo-distribution that minimizes $\Norm{\pE_{\tzeta}[w]}_2^2$ while satisfying $\cA$ on $Y'$ with outlier rate $\eta'$ (computed in Step 3 of the algorithm on input $Y'$). Suppose $Y$ and $Y'$ differ on $i$-th sample point. Let $\tzeta_{adj}$ be the adjacent pseudo-distribution obtained by zeroing out $w_i$. Then, from~\Cref{lem:adjacent-pseudo-distribution}, we know that $\tzeta_{adj}$ is consistent with $\cA$ on input $Y$ with outlier rate $\eta'+ 1/n$. Further, $\Norm{\pE_{\tzeta_{adj}}[w]}_2^2 \leq \Norm{\pE_{\tzeta}[w]}_2^2$. Thus, $\Pot_{\eta' + 1/n}(Y) \leq \Pot_{\eta'}(Y')$. Further,~\Cref{lem:util-score} implies that $\Abs{\Pot_{\eta'+ 1/n}(Y)-\Pot_{\eta'}} \leq 20 L$. Therefore, we have $\Pot_{\eta'}(Y) - \Pot_{\eta'}(Y') \leq 20L$ as desired.

Now, by Cauchy-Schwarz inequality, we immediately obtain that:
\[
\Norm{\pE_{\tzeta}[w]-\pE_{\tzeta'}[w]}_1^2 \leq 20nL
\]
Thus, from Lemma~\ref{lem:normalized-vs-unnormalized}, we have that:
\[
\Norm{p-p'}_1 \leq 120 \sqrt{L/n} \mper
\] 
\end{proof}

\paragraph{Parameter closeness from potential stability}
The following lemma observes that if a sequence of weights $p_i(Y)$ induces a $2k$-certifiably $C'$-subgaussian distribution on $Y$ and $p_i'(Y)$ is a sequence of weights on an adjacent $Y$ such that $p_i(Y)$ is not too far from $p_i(Y')$, then, $p_i(Y')$ must also induce a $2k$-certifiably $C'+1$-subgaussian distribution on $Y'$. 

\begin{lemma} \label{lem:witness-check-succeeds}
Let $0 \leq p_i(Y) \leq \frac{1}{(1-2\eta')}$ be a sequence of non-negative weights adding up to $n$ that induce a $2k$-certifiable $C'$-subgaussian distribution on $Y$. Let $p_i(Y')$ be a sequence of non-negative weights adding up to $n$ on $Y'$ adjacent to $Y$ such that $\Norm{p(Y)-p(Y')}_1 \leq \beta$ for $\beta \leq \eta_0$. Then, for small enough absolute constant $\eta'>0$, $p_i(Y')$ induces a $2k$-certifiable $(C'+1)$-subgaussian distribution on $Y$. 
\end{lemma}
\begin{proof}[Proof Sketch]
Let's first describe the idea of the proof: the proof of Lemma~\ref{lem:witness-property-analysis} requires the existence of a certifiably subgaussian distribution that was close (in total variation distance) to the input $Y$. Since $Y$ is adjacent to $Y'$, the $2k$-certifiably $C'$-subgaussian distribution is $1-\beta-2/n$-close (the additive $2/n$ comes from ``removing'' the index of the point where $Y$ and $Y'$ differ) in total variation distance to $Y$. Thus, the idea is to use the certifiably subgaussian distribution supported on $Y$ in lieu of $X$ to repeat the argument. In order to apply Lemma~\ref{lem:witness-property-analysis}, we need a ``flat'' distribution---but this is easily achieved. Given a distribution with weights (without loss of generality, say, rational numbers $r_i/s$), we can consider a sample expansion to $ns$ samples that has $r_i$ copies of sample $y_i$ for each $i$ and an analogous transformation to $Y'$. And finally, given a pseudo-distribution on $w_1, w_2, \ldots, w_n$ on $Y \cap Y'$, we can transform to a pseudo-distribution on $ns$ variables by each ``copying'' $w_i$ for $i$ such that $y_i = y_i'$ $r_i$ times. 
\end{proof}

As an immediate corollary of Lemma~\ref{lem:pot-stability} and Lemma~\ref{lem:witness-check-succeeds}, we obtain:

\begin{corollary}[Parameter Closeness from Stability of Potential] \label{cor:param-closeness}
Let $\eta, \epsilon, \delta > 0$ and $k, L\in \N$ be given input parameters to \cref{algo:private-robust-moment-estimation} such that $0.25\eta n \geq L = \Omega\left(\frac{1}{\epsilon}\cdot \log\left(\frac{n}{\beta\delta}\right)\right)$. Also, let $Y,Y'$ be adjacent subsets of $\bbQ^d$. 
Suppose \Alg does not reject in any of the 3 steps, uses the constant $C'$ in Step 2 and chooses $\eta'$ in Step~\ref{step:outlier-rate-selection} on input $Y$ and $Y'$. 


Then, for every $u \in \R^d$ and $\theta = \sqrt{L/n}$, we have:
\[
\iprod{\mu_p-\mu_{p'},u} \leq O(C'k)\theta^{1-1/2k} \sqrt{u^{\top} \Sigma_p u} \mcom
\]
\[
(1-O(C'k)\theta^{1-1/k}) \Sigma_{p} \preceq \Sigma_{p'} \preceq (1+O(C'k)\theta^{1-1/k}) \Sigma_{p} \mcom
\]
and, for every $t \leq k$ such that $t$ divides $2k$,
\[
(1-O(C'^{t/2}k^{t/2}))\theta^{1-t/2k}) \iprod{u^{\otimes t}, M^{(t)}_p} \leq \iprod{u^{\otimes t}, M^{(t)}_{p'}} \leq (1+O(C'^{t/2}k^{t/2})\theta^{1-t/2k}) \iprod{u^{\otimes t}, M^{(t)}_p} \mcom
\]
\end{corollary}
\begin{proof}
Let $\tzeta_{adj}$ be the adjacent pseudo-distribution of degree $2k$ to $\tzeta$ obtained by zeroing out $w_i$ where $i$ is the index of the point that $Y$ and $Y'$ differ on. Then, from~\Cref{lem:adjacent-pseudo-distribution}, we know that $\tzeta_{adj}$ satisfies $\cA$ on both inputs $Y,Y'$ with outlier rate $\eta'+1/n$ and $|\Norm{\pE_{\tzeta_{adj}}[w]-\pE_{\tzeta}[w]}_2^2|\leq 1$, $|\Norm{\pE_{\tzeta_{adj}}[w]-\pE_{\tzeta'}[w]}_2^2|\leq 1$. Let $p_{adj}$ be the scaling of $\pE_{\tzeta_{adj}}[w]$ so that $\norm{p_{adj}}_1 = 1$. Then, clearly, $\norm{p-p_{adj}}_1 \leq 2/n$ (since $\eta' \ll 1/2$). Further, applying Lemma~\ref{lem:pot-stability} and triangle inequality, we have that $\norm{p_{adj}-p'}_1 \leq O(\sqrt{L/n})$. Applying Fact~\ref{fact:param-closeness-from-tv-closeness} to $p_{adj}$ and $p$ on $Y$ and $p_{adj}$ and $p'$ on $Y'$ and using triangle inequality completes the proof.  
\end{proof}

\paragraph{Noise injection in estimate-dependent norms} Our final ingredient for obtaining privacy guarantees for our robust estimation algorithms is a new noise injection mechanism where the distribution of noise depends on the covariance estimated by our algorithm. 

\begin{lemma} \label{lem:gaussianhockey}
 Suppose $\epsilon, \delta > 0$. Let $A$ be an invertible $d\times d$ matrix that satisfies $(1-\beta)I \preceq A A^T \preceq (1+\beta) I$, where $\beta \leq \frac{\epsilon}{3d\ln(d/\delta)}$. Let $z\in\R^d$ be a vector whose entries are i.i.d. from $\cN(0,1)$. Then,
 \[
  D_{e^\epsilon}(z, Az) \leq \delta.
 \]
\end{lemma}
\begin{proof}
Note that the probability distribution function of $Az$ at $u\in\R^d$ is
\[
 \frac{1}{\det(A)} \frac{1}{(\sqrt{2\pi})^d} e^{-\Norm{A^{-1} u}_2^2/2}.
\]
Moreover, $\det(A) \leq (1+\beta)^{d/2}$, since $\det(A)^2 = \det(A)\det(A^T) = \det(AA^T) \leq (1+\beta)^d$. Thus, the ratio of the probability densities of $z$ and $Az$ at $u$ is
\begin{align*}
 \det(A) e^{\Norm{A^{-1} u}_2^2/2 - \Norm{u}_2^2/2} &\leq (1+\beta)^{d/2} e^{\Norm{A^{-1} u}_2^2/2 - \Norm{u}_2^2/2}\\
 &\leq (1+\beta)^{d/2} e^{\frac{1}{2} u^T ((AA^T)^{-1} - I) u}\\
 &\leq (1+\beta)^{d/2} e^{\frac{1}{2}(1-\beta)^{-1}\Norm{u}_2^2 - \frac{1}{2}\Norm{u}_2^2}\\
 &\leq (1+\beta)^{d/2} e^{\frac{\beta}{2(1-\beta)}\Norm{u}_2^2}.
\end{align*}
Thus, note that if $\Norm{u}_\infty \leq \sqrt{2 \ln(d/\delta)}$, then $\Norm{u}_2 \leq \sqrt{d} \cdot \Norm{u}_\infty \leq \sqrt{2d \ln(d/\delta)}$, and so,
\[
 \det(A) e^{\Norm{A^{-1}u}_2^2/2 - \Norm{u}_2^2/2} \leq (1+\beta)^{d/2} e^{\frac{\beta}{1-\beta} d\ln(d/\delta)} < e^{\epsilon},
\]
since $\beta \leq \frac{\epsilon}{3d\ln(d/\delta)}$.

Moreover, by standard tail bounds of the normal distribution, we have that $\Norm{z}_\infty > \sqrt{2\ln(d/\delta)}$ with probability at most $\delta$. This proves the claim. \pasin{This seems a tad strange; can't we use tail bound directly on Euclidean norm, and not have a $d$ inside $\ln$ factor?}\avnote{Indeed, I think we can use concentration of norm properties to get that $\Norm{z}_2 < c(\sqrt{d} + \sqrt{\ln(1/\delta)})$ with probability $\geq 1-\delta$, and this will allow us to get away with $\beta < \frac{\epsilon}{d + \ln(1/\delta)}$ instead of $\beta < \frac{\epsilon}{d\ln(d/\delta)}$. In the end, this will give a looser upper bound condition on $\gamma_2$ that needs to be satisfied (see the sentence before \eqref{eq:meandiv2}); however, even with this improvement, the bottleneck will still be the condition on $\gamma_2$ coming from the covariance noise (see \eqref{eq:divergence2}).} 
\end{proof}

\begin{lemma} \label{lem:hockeysticktensored}
 Suppose $\epsilon, \delta > 0$. Let $A$ be a $d\times d$ matrix that satisfies $\Norm{AA^T-I}_2 \leq \beta$.
 
 Let $t \in \N$. Moreover, let $Z\in\R^{d^t}$ be a random vector indexed by $[d]^t$, whose entries $Z_{i_1,i_2,\dots, i_t}$, for $1\leq i_1\leq i_2\leq \cdots \leq i_t \leq d$, are i.i.d. from $\cN(0,1)$, and moreover, $Z_{i_1,i_2,\dots,i_t} = Z_{i_{\pi(1)}, i_{\pi(2)}, \dots, i_{\pi(t)}}$ for any $i=(i_1,i_2,\dots, i_d)$ and permutation $\pi$.
 
 If $\beta \leq \frac{\epsilon}{8t^2 d^t \ln(d^t/\delta)}$, then
 \[
  D_{e^{\epsilon}}(Z, A^{\otimes t} Z) \leq \delta.
 \]
\end{lemma}

\begin{proof}
 Let $K = \sqrt{2 \ln(d^t/\delta)}$. By standard tail bounds, note that
\begin{equation} 
 \Pr[\Norm{Z}_\infty > K] \leq \delta. \label{eq:ztail}
\end{equation} 
 Let $S$ be the subspace of $\R^{d^t}$ consisting of all symmetric tensors, i.e.,
 \[
 S = \left\{u\in\R^{d^t}: u_{i_1, i_2, \dots, i_t} = u_{\left(i_{\pi(1)}, i_{\pi(2)}, \dots, i_{\pi(t)}\right)}, \,\forall i = (i_1,i_2,\dots,i_t)\in [d]^t, \text{$\pi$ a permutation on $[t]$}\right\}.
 \]
 Note that $S$ is an $d'$-dimension invariant subspace of $A^{\otimes t}$, where $d' = {d+t-1 \choose t} \leq d^t$. Moreover, let $R \subseteq [d]^t$ be a representative set of indices of size $|R| = d'$, i.e., $R$ satisfies the property that for any $(i_1,i_2,\dots, i_t)\in [d]^t$, there exists a permutation $\pi$ on $[t]$ such that $(i_{\pi(1)}, i_{\pi(2)}, \dots, i_{\pi(t)}) \in R$.
 
 Now, let $M = A^{\otimes t}|_S$ be the restriction of $A^{\otimes t}$ to the subspace $S$. Moreover, let $Z_R \in \R^{d'}$ denote the projection of $Z$ to indices in $R$.
 
 Note that the probability distribution of $Z$ can be equivalently viewed as the probability distribution of $Z_R$, since $Z$ is uniquely determined by the projection $Z_R$. Let $p$ be the probability density function of $Z_R$ over $\R^{d'}$. Then, note that the probability distribution of $M Z_R$ is $q$, given by 
 \[
  q(v) = \frac{1}{\det(M)} \cdot p(M^{-1} v).
 \]
 for $v\in\R^{d'}$. By standard properties, we know that the $i^\text{th}$ singular value of $M$ is bounded from above by the $i^\text{th}$ singular value of $A^{\otimes t}$ and bounded from below by the $(i+d^t-d')^\text{th}$ singular value of $A^{\otimes t}$.  Moreover, by $\Norm{AA^T-I}_2\leq \beta$, we know that the singular values of $A^{\otimes t}$ lie in $[(1-\beta)^{t/2}, (1+\beta)^{t/2}]$.  Hence, the singular values of $M$ also lie in $[(1-\beta)^{t/2}, (1+\beta)^{t/2}]$, which, together with $\beta t\leq \frac{1}{4}$, implies that
 \begin{equation}
  \Norm{MM^T-I}_2 \leq 2\beta t \label{eq:mmtbd}
 \end{equation}
 and so,
 \begin{equation}
  \det(M)^2 = \det(M)\det(M^T) = \det(MM^T) \leq (1+2\beta t)^{t\cdot d'} \leq (1+2\beta t)^{td^t}, \label{eq:detm-bound}
 \end{equation}
and so, $\det(M) \leq (1+2\beta t)^{td^t/2}$.
 
Let $u\in\R^{d^t}$. Note that $\|u\|_\infty\leq K$ if and only if $v\in\R^{d'}$ given by $v = u|_R$ also satisfies $\|v\|_\infty\leq K$. Moreover, note that if $\|v\|_\infty \leq K$, then
 \begin{align}
  \frac{p(v)}{q(v)} &\leq \det(M) \cdot \frac{p(v)}{p(M^{-1} v)} \nonumber\\
  &\leq (1+2\beta t)^{td^t/2} \cdot \exp\left(\frac{1}{2}\left( \Norm{M^{-1} v}_2^2 - \Norm{v}_2^2 \right)\right) \nonumber\\
  &\leq e^{\beta t^2 d^t} \cdot \exp\left(\frac{1}{2}\left( v^T ((MM^T)^{-1} - I) v\right)\right) \nonumber\\
  &\leq e^{\beta t^2 d^t} \cdot \exp\left(\frac{1}{2}\Norm{v}_2^2 \cdot \Norm{(MM^T)^{-1} - I}_2\right) \label{eq:pdfratio}
 \end{align}\
By \eqref{eq:mmtbd}, we have that $\Norm{(MM^T)^{-1} - I}_2 \leq \frac{2\beta t}{1-2\beta t} \leq 4\beta t$, since $\beta t \leq \frac{1}{4}$. Therefore, \eqref{eq:pdfratio} is at most
\[
 e^{\beta t^2 d^t} \cdot \exp\left(2K^2 d^t \beta t\right).
\]
Thus, if $\beta \leq \frac{\epsilon}{2K^2 t d^t} = \frac{\epsilon}{8t^2 d^{t} \ln(d^t/\delta)}$, the above quantity is at most $e^{\epsilon}$. This, combined with \eqref{eq:ztail}, proves the desired claim.
\end{proof}

\begin{remark}
 Note that \cref{lem:hockeysticktensored} uses an assumption on the \emph{spectral norm} of $\Norm{AA^T-I}_2$. However, it is also possible to obtain a version of the lemma under an assumption on the \emph{Frobenius norm}, $\Norm{AA^T-I}_F$. In particular, if we assume that, instead, $\Norm{AA^T-I}_F \leq \beta$, then \cref{eq:detm-bound} instead becomes $\det(M) \leq \left(1 + \frac{\beta}{\sqrt{d}}\right)^{td^t/2} \leq e^{\beta t d^{t-\frac{1}{2}}/2}$: This follows from the fact that (a.) the eigenvalues $\lambda_1, \lambda_2, \dots, \lambda_d$ of $AA^T$ satisfy $\sum_{i=1}^d (\lambda_i - 1)^2 \leq \beta^2$, (b.) under the aforementioned constraint, $\lambda_1 \lambda_2 \cdots \lambda_d$ is maximized when $\lambda_1 = \lambda_2 = \cdots = \lambda_d = 1 + \frac{\beta}{\sqrt{d}}$, (c.) the eigenvalues of $(AA^T)^{\otimes t}$ are precisely the $d^t$ $t$-fold products of eigenvalues of $AA^T$.
\end{remark}

\paragraph{Putting things together}

Now, we are ready to prove the main privacy guarantee provided by our robust moment estimation algorithm, \cref{algo:private-robust-moment-estimation}.

\begin{lemma}[Privacy Guarantee] \label{lem:privacy-guarantee}
Suppose $C,\eta,\epsilon,\delta > 0$ and $k\in\N$. Suppose $n\geq n_0 = \widetilde{\Omega}\left(\left(\frac{Ck^4 d^k}{\epsilon}\left(\ln(6kd^k/\delta) + \frac{\epsilon}{6k}\right)\right)^\frac{2k}{k-1}\right)$. Then, \Alg\, (given by \cref{algo:private-robust-moment-estimation}), invoked with $L = O(\log(n/\delta)/\epsilon)$, is $(\epsilon,\delta)$-DP.
\end{lemma}

\begin{proof}
Let $\epsilon' = \epsilon/3$ and $\delta' = \delta/3$. By our adaptive composition theorem under halting (\Cref{lem:composition}), it suffices to show that each step of the algorithm is $(\eps', \delta')$-DP (given the outputs of the previous steps as parameter\footnote{Note that we may also assume that the algorithm has not halted from the previous steps.}). Let $Y$ and $Y'$ be any neighboring datasets.
\begin{itemize}
\item  \textbf{Stable Outlier Rate Selection.} Since this step invokes the $(\epsilon', \delta')$-DP Selection algorithm (\SelectionAlg), it immediately follows from \Cref{thm:dp-apx-selection} that this step is $(\epsilon', \delta')$-DP.
\item \textbf{Witness Checking.} \pasin{I'd suggested phrasing the DP guarantee of this step and the next step in terms of~\Cref{def:dp-cond} because our composition theorem is stated in that term.} Let $C^*(Y)$ denote the smallest $C^*$ for which $p_i(Y)$ induces a $2k$-certifiable $C^*$-subgaussian distribution on $Y$. \Cref{lem:witness-check-succeeds} ensures that $|C^*(Y) - C^*(Y')| \leq \Delta$ for $\Delta =1$. Therefore, we may apply \Cref{lem:tlap-dp} with DP parameters $\epsilon', \delta'$ to conclude that this step is also $(\eps', \delta')$-DP.
\item \textbf{Noise Addition.} Since the algorithm has not halted in the previous step and the truncated Laplace noise is negative, $p_i(Y)$ and $p_i(Y')$ must induce $2k$-certifiable $C'$-subgaussian distributions on $Y$ and $Y'$ respectively. Let $\widetilde{\mu}$ and $\widetilde{\mu}'$ denote the corresponding mean estimates under $p_i(Y)$ and $p_i(Y')$, respectively, and, similarly, let $\widetilde{\Sigma}$ and $\widetilde{\Sigma}'$ denote the corresponding covariance estimates. By \Cref{cor:param-closeness}, we have that, for all $u\in\R^d$,
\begin{align}
 \iprod{\widetilde{\mu} -\widetilde{\mu}', u} &\leq \gamma_1 \sqrt{u^{\top} \widetilde{\Sigma} u} \label{eq:meandotprod}\\
(1-\gamma_2) \widetilde{\Sigma} \preceq \widetilde{\Sigma}' &\preceq (1+\gamma_2) \widetilde{\Sigma} \label{eq:covarbound}
\end{align}
and, for all $2\leq t \leq k$,
\begin{equation}
 (1-\gamma_t) \iprod{u^{\otimes t}, \widetilde{M^{(t)}}} \leq \iprod{u^{\otimes t}, \widetilde{M'^{(t)}}} \leq (1+\gamma_t) \iprod{u^{\otimes t}, \widetilde{M^{(t)}}}, \label{eq:momentbound}
\end{equation}
where $\theta = \sqrt{L/n}$, $\gamma_1 = O(C'k)\theta^{1-1/2k}$, and $\gamma_t = O((C'k)^{t/2})\theta^{1-t/2k}$ for $2 \leq t\leq k$. Moreover, let $B = \widetilde{\Sigma}^{-1/2} \widetilde{\Sigma}'^{1/2}$.

Note that in order to show that the noise addition step is $(\epsilon',\delta')$-DP, it suffices to show that
\begin{align}
 D_{e^{\epsilon''}} (\widetilde{\mu} + \widetilde{\Sigma}^{1/2} z, \widetilde{\mu}' + \widetilde{\Sigma}'^{1/2} z) &\leq \delta'' \label{eq:meandp}\\
 D_{e^{\epsilon''}}(\widetilde{\Sigma} + \widetilde{\Sigma}^{1/2} Z \widetilde{\Sigma}^{1/2}, \widetilde{\Sigma}' + \widetilde{\Sigma}'^{1/2} Z \widetilde{\Sigma}'^{1/2}) &\leq \delta'' \label{eq:covardp}\\
 \forall\, 2 < t \leq k,\quad D_{e^{\epsilon''}}(\widetilde{M^{(t)}} + (\widetilde{\Sigma}^{1/2})^{\otimes t} Z^{(t)}, \widetilde{M'^{(t)}} + (\widetilde{\Sigma}'^{1/2})^{\otimes t} Z^{(t)}) &\leq \delta'' \label{eq:momentdp}
\end{align}
for $\epsilon'' = \epsilon'/k$ and $\delta'' = \delta'/k$, since \cref{fact:dphockey} and standard DP composition~\cite{DworkKMMN06} then imply that the entire noise addition step is $(\epsilon', \delta')$-DP. We now establish each of the above inequalities.

\textbf{Noise addition for mean:} We first show \eqref{eq:meandp}. Note that
\begin{align}
 D_{e^{\epsilon''}} (\widetilde{\mu} + \widetilde{\Sigma}^{1/2} z, \widetilde{\mu}' + \widetilde{\Sigma}'^{1/2} z) &= D_{e^{\epsilon''}}(\widetilde{\Sigma}^{1/2} z, (\widetilde{\mu}' - \widetilde{\mu}) + \widetilde{\Sigma}'^{1/2} z) \nonumber\\
 &= D_{e^{\epsilon''}}(z, \widetilde{\Sigma}^{-1/2}(\widetilde{\mu}'-\widetilde{\mu}) + Bz) \nonumber\\
 &= D_{e^{\epsilon''/2}}(z, z + \widetilde{\Sigma}^{-1/2}(\widetilde{\mu}'-\widetilde{\mu})) \nonumber\\
 &\quad + e^{\epsilon''/2} D_{e^{\epsilon''/2}}(z + \widetilde{\Sigma}^{-1/2}(\widetilde{\mu}'-\widetilde{\mu}), \widetilde{\Sigma}^{-1/2}(\widetilde{\mu}'-\widetilde{\mu}) + Bz) \label{eq:meantriangle} \\
 &= D_{e^{\epsilon''/2}}(z, z + \widetilde{\Sigma}^{-1/2}(\widetilde{\mu}'-\widetilde{\mu})) + D_{e^{\epsilon''/2}}(z, Bz), \label{eq:meanhockey}
\end{align}
where \eqref{eq:meantriangle} follows from \cref{lem:hockeytriangle}. For the first term on the right-hand side of \eqref{eq:meanhockey}, we note that $\Norm{\widetilde{\Sigma}^{-1/2}(\widetilde{\mu}' - \widetilde{\mu})}_2 \leq \gamma_1$ (which follows from plugging in $u = \widetilde{\Sigma}^{-1}(\widetilde{\mu}-\widetilde{\mu}')$ into \eqref{eq:meandotprod}). Thus, by the standard hockey-stick divergence calculation for the Gaussian mechanism~\cite[Appendix A]{DworkR14}, we have that
\begin{equation}
 D_{e^{\epsilon''/2}}(z, z + \widetilde{\Sigma}^{-1/2}(\widetilde{\mu}'-\widetilde{\mu})) < \delta''/2, \label{eq:meandiv1}
\end{equation}
provided that
\[
 \sigma_1 \geq \frac{2\gamma_1 \sqrt{2\ln(2.5/\delta'')}}{\epsilon''},
\]
For the second term in \eqref{eq:meanhockey}, note that \eqref{eq:covarbound} implies that $(1-\gamma_2) I \preceq BB^T \preceq (1+\gamma_2)I$. Moreover, $\gamma_2 \leq \frac{\epsilon''}{3d\ln(2d/\delta'')}$, by the condition $n\geq n_0$. Thus, by \cref{lem:gaussianhockey},
\begin{equation}
 D_{e^{\epsilon'/2}}(z,Bz) \leq \delta''/2. \label{eq:meandiv2}
\end{equation}
Therefore, \eqref{eq:meandiv1}, \eqref{eq:meandiv2},and \eqref{eq:meanhockey} imply \eqref{eq:meandp}, as desired.
\end{itemize} 

\textbf{Noise addition for covariance:}
Next, we establish \eqref{eq:covardp}. Observe that 
\begin{align}
 D_{e^{\epsilon''}}(\widetilde{\Sigma} + \widetilde{\Sigma}^{1/2} Z \widetilde{\Sigma}^{1/2}, \widetilde{\Sigma}' + \widetilde{\Sigma}'^{1/2} Z \widetilde{\Sigma}'^{1/2}) &= D_{e^{\epsilon''}} (I+Z, BB^T + BZB^T)\nonumber\\
 &= D_{e^{\epsilon''}}(Z, (BB^T-I) + BZB^T)\nonumber\\
 &\leq D_{e^{\epsilon''/2}}(Z, Z + (BB^T-I))\nonumber\\
 &\quad + e^{\epsilon''/2} D_{e^{\epsilon''/2}}(Z + (BB^T-I), (BB^T-I) + BZB^T)) \label{eq:hockeytriangleineq} \\
 &\leq D_{e^{\epsilon''/2}}(Z, Z + (BB^T-I)) + e^{\epsilon''/2} D_{e^{\epsilon''/2}}(Z, BZB^T)), \label{eq:hockeytwoterms}
\end{align}
where \eqref{eq:hockeytriangleineq} follows from \cref{lem:hockeytriangle}. To bound the right-hand side of \eqref{eq:hockeytwoterms}, note that the first term is precisely the hockey-stick divergence computation corresponding to the Gaussian mechanism (restricted to the upper triangular portion). Moreover, by \eqref{eq:covarbound},
\begin{equation}
\Norm{BB^T-I}_F \leq \sqrt{d} \cdot \Norm{BB^T-I}_2 \leq \gamma_2\sqrt{d}. \label{eq:bfrob}
\end{equation}
Therefore (\cite[Appendix A]{DworkR14}), as long as
\[
 \sigma_2 \geq \frac{2\gamma_2\sqrt{2d\ln(2.5/\delta'')}}{\epsilon''},
\]
we have that
\begin{equation}
 D_{e^{\epsilon''/2}}(Z, Z + (BB^T-I)) \leq \delta''/2. \label{eq:divergence1}
\end{equation}
For the second term in \eqref{eq:hockeytwoterms}, note that $\sigma_2^{-1} Z$ has entries distributed in $\cN(0,1)$. Moreover, let $Z' = \vectorize(Z)$ be the $d^2$-dimensional vector given by the flattening of $Z$ (see \cref{def:flattening}). By \cref{fact:mixedprod}, we know that $BZB^T = B^{\otimes 2} Z'$. Thus, by \cref{lem:hockeysticktensored} applied with $t = 2$,
\begin{align}
 D_{e^{\epsilon''/2}}(Z, BZB^T)) &= D_{e^{\epsilon''/2}}(Z', B^{\otimes 2}Z') \nonumber\\
 &= D_{e^{\epsilon''/2}}(\sigma_2^{-1} Z', B^{\otimes 2}(\sigma_2^{-1} Z')) \nonumber\\
 &\leq \delta''/2e^{\epsilon''/2}, \label{eq:divergence2}
\end{align}
as long as
\[
 \gamma_2 < \frac{\epsilon''}{32d^2 \ln(2d^2 e^{\epsilon''/2}/\delta'')},
\]
which is true, since $n\geq n_0$ by the conditions of the theorem. Thus, \eqref{eq:divergence1} and \eqref{eq:divergence2} imply that \eqref{eq:hockeytwoterms} is at most $\delta''/2 + e^{\epsilon''/2} (\delta''/2e^{\epsilon''/2}) = \delta''/2$, which establishes \eqref{eq:covardp}.

\textbf{Noise addition for higher-order moments:}
Let $2 < t \leq k$. We write $R = \widetilde{\Sigma} + \widetilde{\mu}\widetilde{\mu}^T$ and $R' = \widetilde{\Sigma}' + \widetilde{\mu}'\widetilde{\mu}'^T$ for simplicity.

Observe that the injective/spectral norm $\Norm{\cdot}_\sigma$ of $(R^{-1/2})^{\otimes t} \left( \widetilde{M'^{(t)}} - \widetilde{M^{(t)}}\right)$ can be bounded as
\begin{align}
 \Norm{(R^{-1/2})^{\otimes t} \left( \widetilde{M'^{(t)}} - \widetilde{M^{(t)}}\right)}_\sigma &= \sup_{\substack{v\in\R^d\\ \Norm{v}_2=1}} \Abs{\left(v^{\otimes t}\right)^T (R^{-1/2})^{\otimes t} \left( \widetilde{M'^{(t)}} - \widetilde{M^{(t)}}\right)} \nonumber\\
 &\leq \sup_{\substack{v\in\R^d\\ \Norm{v}_2=1}} \Abs{\iprod{(R^{-1/2} v)^{\otimes t}, \widetilde{M'^{(t)}} - \widetilde{M^{(t)}}}} \nonumber\\
 &\leq \gamma_t \cdot \sup_{\substack{v\in\R^d\\ \Norm{v}_2=1}} \Abs{\iprod{(R^{-1/2} v)^{\otimes t}, \widetilde{M^{(t)}}}} \nonumber \\
 &\leq \gamma_t \cdot (C'k)^t \cdot \sup_{\substack{v\in\R^d\\ \Norm{v}_2=1}} \Abs{\left(\left(R^{-1/2}v\right)^T R \left(R^{-1/2} v\right)\right)^{t/2}} \label{eq:subgaussian-transform} \\
 &= \gamma_t \cdot (C'k)^t \cdot \sup_{\substack{v\in\R^d\\ \Norm{v}_2=1}} \Norm{v}_2^t \nonumber\\
 &= \gamma_t \cdot (C'k)^t, \nonumber
\end{align}
where \eqref{eq:subgaussian-transform} follows from the $C'$-subgaussianity property of the distribution induced by the weight vector at the end of Step~\ref{step:witness-checking}. Therefore, the Frobenius norm (or Hilbert-Schmidt norm) can be bounded as (see Corollary 4.10 of \cite{Wang2017OPERATORNI})
\begin{equation}
 \Norm{(R^{-1/2})^{\otimes t} \left( \widetilde{M'^{(t)}} - \widetilde{M^{(t)}}\right)}_F \leq d^{\frac{t-1}{2}} \cdot \Norm{(R^{-1/2})^{\otimes t} \left( \widetilde{M'^{(t)}} - \widetilde{M^{(t)}}\right)}_\sigma \leq \gamma_t\cdot (C'k)^t\cdot d^{\frac{t-1}{2}}. \label{eq:frobmomentdiff}
\end{equation}
Moreover, letting $W = R^{-1/2} R'^{1/2}$, we have
\begin{align}
 D_{e^{\epsilon''}}(\widetilde{M^{(t)}} + (R^{1/2})^{\otimes t} Z^{(t)}, \widetilde{M'^{(t)}} + (R'^{1/2})^{\otimes t} Z^{(t)}) &= D_{e^{\epsilon''}} \left((R^{1/2})^{\otimes t} Z^{(t)}, \widetilde{M'^{(t)}} - \widetilde{M^{(t)}} + (R'^{1/2})^{\otimes t} Z^{(t)}\right)\nonumber\\
 &= D_{e^{\epsilon''}}\left(Z^{(t)}, (R^{-1/2})^{\otimes t} \left( \widetilde{M'^{(t)}} - \widetilde{M^{(t)}} \right) + W^{\otimes t} Z^{(t)} \right)\nonumber\\
 &\leq D_{e^{\epsilon''/2}}\left(Z^{(t)}, Z^{(t)} + (R^{-1/2})^{\otimes t} \left( \widetilde{M'^{(t)}} - \widetilde{M^{(t)}} \right) \right)\nonumber\\
 &\quad + e^{\epsilon''/2} D_{e^{\epsilon''/2}}( Z^{(t)} + (R^{-1/2})^{\otimes t} \left( \widetilde{M'^{(t)}} - \widetilde{M^{(t)}} \right), \nonumber\\
 &\quad (R^{-1/2})^{\otimes t} \left( \widetilde{M'^{(t)}} - \widetilde{M^{(t)}} \right) + W^{\otimes t} Z^{(t)})) \label{eq:momenthockeytriangleineq} \\
 &\leq D_{e^{\epsilon''/2}}\left(Z^{(t)}, Z^{(t)} + (R^{-1/2})^{\otimes t} \left( \widetilde{M'^{(t)}} - \widetilde{M^{(t)}} \right)\right) \nonumber\\
 &\quad + e^{\epsilon''/2} D_{e^{\epsilon''/2}}(Z^{(t)}, W^{\otimes t} Z^{(t)}), \label{eq:momenthockeytwoterms}
\end{align}
where again we have used \cref{lem:hockeytriangle} in \eqref{eq:momenthockeytriangleineq}. In order to bound the right-hand side of \eqref{eq:momenthockeytwoterms}, note that the first term is again the hockey-stick divergence computation corresponding to the Gaussian mechanism (restricted according to symmetry conditions). Recalling \eqref{eq:frobmomentdiff}, we see that (\cite[Appendix A]{DworkR14}) as long as 
\[
 \sigma_t \geq \frac{2 \gamma_t (C'k)^t d^{\frac{t-1}{2}} \sqrt{2\ln(2.5/\delta'')}}{\epsilon''},
\]
we have that
\begin{equation}
 D_{e^{\epsilon''/2}}\left(Z^{(t)}, Z^{(t)} + (R^{-1/2})^{\otimes t} \left( \widetilde{M'^{(t)}} - \widetilde{M^{(t)}} \right)\right) \leq \delta''/2. \label{eq:momentdivergence1}
\end{equation}
For the second term in \eqref{eq:momenthockeytwoterms}, note that $\sigma_t^{-1} Z^{(t)}$ has entries distributed in $\cN(0,1)$. Moreover, note that $\Norm{WW^T-I}_F \leq \Norm{BB^T-I}_F \leq \gamma_2\sqrt{d}$ by \eqref{eq:bfrob} and the fact that \eqref{eq:covarbound} implies
\[
 (1-\gamma_2)R \preceq R' \preceq (1+\gamma_2)R.
\]
Thus, by \cref{lem:hockeysticktensored}, we have that
\begin{align}
 D_{e^{\epsilon''/2}} \left(Z^{(t)}, W^{\otimes t} Z^{(t)} \right) &= D_{e^{\epsilon''/2}} \left(\sigma_t^{-1} Z^{(t)}, W^{\otimes t} (\sigma_t^{-1} Z^{(t)})\right) \nonumber\\
 &\leq \delta''/2e^{\epsilon''/2}, \label{eq:momentdivergence2}
\end{align}
as long as
\[
 \gamma_2 < \frac{\epsilon''}{16t^2 d^t \ln(2d^t e^{\epsilon''/2}/\delta'')},
\]
which is true since $n\geq n_0$, by the hypothesis of the lemma. Thus, \eqref{eq:momentdivergence1} and \eqref{eq:momentdivergence2} imply that \eqref{eq:momenthockeytwoterms} is at most $\delta''/2 + e^{\epsilon''/2} (\delta''/2e^{\epsilon''/2}) = \delta''$, which establishes \eqref{eq:momentdp}, as desired.
\end{proof}

\subsection{Proof of \cref{thm:diff-priv-robust-moment-estimation-section}} \label{subsec:mainproof}
We are now ready to prove our main theorem, \cref{thm:diff-priv-robust-moment-estimation-section}.
\begin{proof}[Proof of \cref{thm:diff-priv-robust-moment-estimation-section}]
 Choose $\beta = 1/30$. Choose $L = \Omega\left(\frac{1}{\epsilon} \cdot \log\left(\frac{n}{\beta\delta}\right)\right)$ (according to the condition in \cref{lem:util-score}). Moreover, let $C = C_0 + \frac{3 \ln(3/\delta)}{\epsilon} + \frac{9}{\epsilon} + 1$. Then, we claim that setting $\Alg$ to be \cref{algo:private-robust-moment-estimation} with parameters $C, \eta, \epsilon,\delta, L, k$ satisfies the desired conditions, as long as $\eta \leq \eta_0$, where we set $\eta_0$ later.
 
 Note that the desired privacy guarantees follow immediately from \cref{lem:privacy-guarantee}.
 
 It remains to prove the utility guarantees. Suppose that there indeed exists a good set $X\subseteq \bbQ^d$ with mean $\mu_*$, covariance $\Sigma_*$, and $t$-th moments $M_*^{(t)}$ for $3\leq t\leq k$, such that $|Y \cap X| \geq (1-\eta)n$.
 
 By \cref{thm:dp-apx-selection}, we have that Step \ref{step:outlier-rate-selection} (stable outlier rate selection) rejects and halts with probability at most $\beta = 1/30$, and the resulting output $\tau$ satisfies $\score(\tau,Y) \geq L/2$. In particular, the latter condition implies that there exists some $\gamma \geq L/2$ for which $\cA\left(\frac{\tau-\gamma}{n}\right)$ is feasible. By monotonicity, $\cA\left(\eta'\right)$ is also feasible, where we let $\eta' = \tau/n$.
 
 Hence, the invocation of \cref{algo:robust-moment-estimation-witness-producing} in Step~\ref{step:witness-checking} does not yield ``reject.''  Moreover, note that by \cref{lem:tlap-tail}, we have that $C' = C+\gamma \geq C_0$ with probability at least $29/30$. In this case, the computed weight vector $p$ induces a $C'$-certifiably subgaussian distribution on $Y$. Hence, the probability of rejection in Step~\ref{step:witness-checking} is at most $1/30$.
 
 Let $\widetilde{\mu} = \pE_{\tzeta}[\mu]$, $\widetilde{\Sigma} = \pE_{\tzeta}[\Sigma]$, and $\widetilde{M}^{(t)} = \pE_{\tzeta}[\Sigma]$ (for $2\leq t\leq k$) be the estimates of the mean, covariance, and $t$-th moments, respectively, that are outputted by the \cref{algo:robust-moment-estimation-witness-producing} subroutine in Step~\ref{step:witness-checking} of \cref{algo:private-robust-moment-estimation}. Then, by \cref{lem:witness-producing-robust-moment-estimation}, we have
 \begin{align*}
  \forall u \in \R^d, \text{  } \iprod{\widetilde{\mu}-\mu_*,u} &\leq O\left(\sqrt{Ck}\right) \eta^{1-1/2k} \sqrt{ u^{\top} \Sigma_*u}\\
  (1-\beta_2) \Sigma_* &\preceq \widetilde{\Sigma} \preceq (1+\beta_2)\Sigma_*
 \end{align*}
 and, for all even $2\leq t\leq k$ such that $t$ divides $2k$,
 \[
  \forall u \in \R^d, \text{  } (1-\beta_t) \iprod {u^{\otimes t}, M^{(t)}_*} \leq  \iprod{u^{\otimes t}, \widetilde{M}^{(t)}} \leq (1+\beta_t) \iprod {u^{\otimes t}, M^{(t)}_*},
 \]
where $\beta_t = \beta_t(\eta) = O((Ck)^{t/2}) \eta^{1-t/2k}$. We now set $\eta_0$ such that $\beta_t(\eta_0)\leq \frac{1}{2}$ for all aforementioned $t$. Note that this guarantees that $\beta_t = \beta_t(\eta) \leq \frac{1}{2}$, since we are assuming $\eta\leq\eta_0$.

Now, consider the noise addition step, i.e., Step~\ref{step:noise-addition} of \cref{algo:private-robust-moment-estimation}. Note that by the Cauchy-Schwarz Inequality, for any $u\in\R^d$, we have
\begin{align*}
 \iprod{\widetilde{\Sigma}^{1/2}z, u} &= \iprod{\Sigma_*^{-1/2} \widetilde{\Sigma}^{1/2}z, \Sigma_*^{1/2} u}\\
 &\leq \Norm{\Sigma_*^{-1/2} \widetilde{\Sigma}^{1/2}z}_2 \cdot \Norm{\Sigma_*^{1/2} u}_2\\
 &= (z^T \widetilde{\Sigma}^{1/2} \Sigma_*^{-1} \widetilde{\Sigma}^{1/2} z)\cdot \sqrt{ u^{\top} \Sigma_*u}\\
 &\leq \Norm{z}_2^2 \cdot \left(1 + \Norm{\widetilde{\Sigma}^{1/2} \Sigma_*^{-1} \widetilde{\Sigma}^{1/2} - I}_2 \right) \cdot \sqrt{ u^{\top} \Sigma_*u}\\
 &\leq \Norm{z}_2^2 \cdot \left(1 + 2\beta_2 \right) \cdot \sqrt{ u^{\top} \Sigma_*u}\\
 &\leq 2 \Norm{z}_2^2 \cdot \sqrt{u^T \Sigma_* u},
\end{align*}
since $\beta_2 \leq \frac{1}{2}$ by our choice of $\eta_0$. Now, note that with probability at least $1 - \frac{1}{30k}$, we have that $\Norm{z}_2 \leq O\left(\sigma_1\sqrt{d\ln(kd)}\right)$, in which case it follows that
\[
 \Norm{z}_2^2 = O(d) \cdot \sigma_1^2 \ln(kd) = O(\sqrt{Ck}) \eta^{1-1/2k},
\]
by our choice of $n\geq n_0$. Thus, the mean estimate $\hat{\mu}$ outputted by the Step~\ref{step:noise-addition} satisfies
\begin{align}
 \iprod{\hat{\mu}-\mu_*, u} &= \iprod{\hat{\mu}-\widetilde{\mu}, u} + \iprod{\widetilde{\mu}-\mu_*, u} \nonumber \\
 &= \iprod{\widetilde{\Sigma}^{1/2}z, u} + \iprod{\widetilde{\mu}-\mu_*, u} \nonumber \\
 &= O\left(\sqrt{Ck}\right)\eta^{1-1/2k}\sqrt{u^T\Sigma_* u}. \label{eq:util-mean}
\end{align}

Next, we consider the utility guarantee for the covariance. Note that $\Norm{Z}_2\leq \nu_2 = O\left(\sigma_2 \sqrt{d \ln(kd^2)}\right)$ with probability at least $1 - \frac{1}{30k}$ (this follows from standard spectral properties of Wigner matrices; see, for instance, \cite{taotopics}), in which case, it follows that $-\nu_2 \widetilde{\Sigma} \preceq \widetilde{\Sigma}^{1/2} Z \widetilde{\Sigma}^{1/2} \preceq \nu_2 \widetilde{\Sigma}$. Moreover, by our choice of $n\geq n_0$ as well as $\eta_0$, we have that $\nu_2 \leq \beta_2 \leq \frac{1}{2}$. Thus, it follows that
\begin{align*}
 \hat{\Sigma} &\preceq (1+\beta_2)\widetilde{\Sigma}\\
  &\preceq (1+\beta_2)^2\Sigma_* \\
  &\preceq \left(1+\frac{5}{2}\beta_2\right)\Sigma_* \\
  &\preceq \left(1 + O(Ck)\cdot \eta^{1-1/k}\right)\Sigma_*,
\end{align*}
and by a similar argument, we also have $\hat{\Sigma} \succeq \left(1-O(Ck)\cdot \eta^{1-1/k}\right)\Sigma_*$, thus implying that
\begin{equation}
 \left(1-O(Ck)\cdot \eta^{1-1/k}\right)\Sigma_* \preceq \hat{\Sigma} \preceq \left(1 + O(Ck)\cdot \eta^{1-1/k}\right)\Sigma_*. \label{eq:util-covariance}
\end{equation}

Finally, we consider the utility guarantee for moment estimation. Suppose $2 \leq t \leq k$ and $t$ is an even number dividing $2k$. Let $A = \widetilde{\Sigma} + \widetilde{\mu}\widetilde{\mu}^T$ and $A_* = \Sigma_* + \mu_* \mu_*^T$. Note that for any $2 < t \leq k$, we have $\Norm{Z^{(t)}}_F = O\left(\sigma_t d^{t/2} \sqrt{\ln(kd^t)}\right)$ with probability at least $1 - \frac{1}{30k}$. In this case, note that for any $u\in\R^d$, we have the following (recall that $\Norm{\cdot}_\sigma$ indicates the injective norm of a tensor):
\begin{align}
 \iprod{u^{\otimes t}, (A^{1/2})^{\otimes t} Z^{(t)}} &= \iprod{(A^{1/2} u)^{\otimes t}, Z^{(t)}} \nonumber \\
 &\leq \Norm{Z^{(t)}}_\sigma \cdot \Norm{A^{1/2} u}_2^t \nonumber \\
 &\leq \Norm{Z^{(t)}}_F \cdot \Norm{A^{1/2} u}_2^t \nonumber \\
 &= O(\sigma_t d^{t/2} \sqrt{\ln(kd^t)}) \cdot \Norm{A^{1/2} u}_2^t \nonumber \\
 &= O(\sigma_t d^{t/2} \sqrt{\ln(kd^t)}) \cdot (u^T A u)^{t/2} \nonumber \\
 &= O(\sigma_t (d^{t/2} \sqrt{\ln(kd^t)}) \cdot ((1+\beta_2)u^T A_* u)^{t/2} \nonumber \\
 &= O(\sigma_t (d(1+\beta_2))^{t/2} \sqrt{\ln(kd^t)}) \cdot \Norm{A_*^{1/2} u}_2^t \nonumber \\
 &= O(\sigma_t (de^{\beta_2})^{t/2} \sqrt{\ln(kd^t)}) \cdot \iprod{u^{\otimes t}, M_*^{(t)}} \label{eq:jensen} \\
 &= O((Ck)^{t/2}) \eta^{1-t/2k } \cdot \iprod{u^{\otimes t}, M_*^{(t)}}, \label{eq:btbound}
\end{align}
where \eqref{eq:jensen} follows from Jensen's Inequality, and \eqref{eq:btbound} follows from our choice of $n\geq n_0$. Thus, the moment estimate $\hat{M}^{(t)} = \widetilde{M}^{(t)} + (A^{1/2})^{\otimes t} Z^{(t)}$ outputted by our algorithm satisfies
\begin{align*}
 \iprod{u^{\otimes t}, \hat{M}^{(t)}} &\leq \iprod{u^{\otimes t}, \widetilde{M}^{(t)}} + \iprod{u^{\otimes t}, (A^{1/2})^{\otimes t} Z^{(t)}}\\
 &\leq (1+\beta_t) \iprod{u^{\otimes t}, M_*^{(t)}} + O((Ck)^{t/2}) \eta^{1-t/2k } \cdot \iprod{u^{\otimes t}, M_*^{(t)}}\\
 &= \left(1 + O((Ck)^{t/2}) \eta^{1-t/2k}\right) \iprod{u^{\otimes t}, M_*^{(t)}}.
\end{align*}
In a similar fashion, we also get that $\iprod{u^{\otimes t}, \hat{M}^{(t)}} \geq \left(1 - O((Ck)^{t/2}) \eta^{1-t/2k}\right) \iprod{u^{\otimes t}, M_*^{(t)}}$, thus implying that
\begin{equation}
  \left(1 - O((Ck)^{t/2}) \eta^{1-t/2k}\right) \iprod{u^{\otimes t}, M_*^{(t)}} \leq \iprod{u^{\otimes t}, \hat{M}^{(t)}} \leq \left(1 + O((Ck)^{t/2}) \eta^{1-t/2k}\right) \iprod{u^{\otimes t}, M_*^{(t)}}. \label{eq:util-moment}
\end{equation}

Hence, \eqref{eq:util-mean}, \eqref{eq:util-covariance}, and \eqref{eq:util-moment} imply the desired utility guarantees.

Moreover, recall that the rejection probabilities at Steps~\ref{step:outlier-rate-selection} and \ref{step:witness-checking} are each at most $\frac{1}{30}$, and it is not possible to reject in Step~\ref{step:noise-addition}. Moreover, the $\leq k$ utility guarantees each fail with probability at most $\frac{1}{30k}$. Thus, by a union bound, it follows that the algorithm does not reject and, moreover, outputs estimates satisfying the desired utility guarantees with probability at least $1 - \frac{1}{30} - \frac{1}{30} - k \cdot \frac{1}{30k} = \frac{9}{10}$.

Finally, note that the running time of $(Bn)^{O(k)}$ follows from the time complexity guarantee in \cref{lem:witness-producing-robust-moment-estimation}, as the invocation of Algorithm~\ref{algo:robust-moment-estimation-witness-producing} in Step~\ref{step:witness-checking} is the bottleneck. Steps~\ref{step:outlier-rate-selection} and \ref{step:noise-addition} are easily seen to run in $(Bn)^{O(k)}$ time. This completes the proof.
\end{proof}

\section{Robust Mean and Covariance Estimation for Certifiably Hypercontractive Distributions}
In this section, we observe that we can upgrade our guarantees from the previous section for robust estimation of moments of distributions that have certifiably hypercontractive degree $2$ polynomials. 

\begin{definition}
A distribution $cD$ on $\R^d$ with mean $\mu_*$ and covariance $\Sigma_*$ is said to have $2h$-certifiably $C$-hypercontractive degree $2$ polynomials if for a $d \times d$ matrix-valued indeterminate $Q$ and $\bar{x}=x-\mu_*$,
\[
\sststile{2h}{Q} \Set{\E_{x \sim D} (\bar{x}^{\top}Q\bar{x}-\E_{x \sim D} \bar{x}^{\top}Q\bar{x})^{2h} \leq (Ch)^{2h} \Norm{\Sigma_*^{1/2}Q\Sigma_*^{1/2}}_F^{2h}}\mper
\]
\end{definition}
The Gaussian distribution~\cite{DBLP:conf/soda/KauersOTZ14}, uniform distribution on the hypercube and more generally other product domains and their affine transforms are known to satisfy $2t$-certfiably $C$-hypercontractivity with an absolute constant $C$ for every $t$. 

In order to derive this conclusion, we note the following analog of the witness-producing algorithm and its guarantees:

\paragraph{Witness-producing version of the robust moment estimation algorithm} 
We will use the following (non-private) guarantees for the robust moment estimation algorithm in the previous section that hold for a strengthening of the constraint system $\cA$ with certifiable hypercontractivity constraints. Using the analysis of~\cite{DKKLMS16}, the following guarantees were recently shown in~\cite{kothari2021polynomialtime} for the case when the unknown distribution is Gaussian. 

For any $d \times d$ matrix-valued indeterminate $Q$, let $\bar{x'_i}^{\top}Q\bar{x'_i} = {x'}^{\top}Q x'- \frac{1}{n} \sum_{i = 1}^n {x_i'}^{\top}Qx'$.

\begin{mdframed}[frametitle={$\cA$: Constraint System for $\eta$-Robust Moment Estimation}]
        \begin{enumerate}
            \item $w_i^2 = w_i$ for each $1 \leq i \leq n$,
            \item $\Pi^2 = \frac{1}{n} \sum_{i = 1}^n ({x'}_i-\mu')({x'}_i-\mu')^{\top}$,
            \item $\sum_{i = 1}^n w_i \geq (1-\eta)n$,
            \item $\mu' = \frac{1}{n} \sum_i x_i'$,  
            \item $w_i (x_i' - y_i) = 0$ for $1 \leq i \leq n$, 
            \item $\frac{1}{n} \sum_{i = 1}^n \bar{x'_i}^{\top}Q\bar{x'_i}^{2} \leq C \Norm{\Pi Q \Pi}_F^{2}$. 
        \end{enumerate}
        \label{box:constraints-hypercontractive}
\end{mdframed} 

The following guarantees for the algorithm above were shown in~\cite{BK20}. 
\begin{fact}[\cite{BK20}] \label{fact:ks-analysis-hypercontractive}
Let $X \subseteq \R^d$ be an i.i.d. sample of size $n \geq n_0 = \widetilde{O}(d^2/\eta)$ from $\cN(\mu_*,\Sigma_*)$. Let $Y$ be an $\eta$-corruption of $X$. Then, for $\mu' = \frac{1}{n} \sum_i x_i'$, $\Sigma' = \frac{1}{n} \sum_i (x_i-\mu')(x_i-\mu')^{\top}$, we have:
\[
\cA \sststile{O(k)}{u} \Set{ \iprod{\mu'-\mu_*,u} \leq O(\eta^{1-1/2k}) u^{\top} \Sigma_* u^{2}}\mcom
\]
\[
\cA \sststile{O(k)}{u} \Set{ \iprod{u,\Sigma'-\Sigma_*,u} \leq O(\eta^{1-1/k}) u^{\top} \Sigma_* u}\mcom
\]
\[
\cA \sststile{O(k)}{} \Set{ \Norm{\Sigma_*^{-1/2} \Sigma' \Sigma_*^{-1/2} -I}_F^2 \leq O(\eta^{1-1/k})}\mper
\]

\end{fact}

The first two guarantees of the lemma below were shown in~\cite{BK20}. The third guarantee follows from an argument similar to that of Lemma~\ref{lem:witness-property-analysis}. Notice that the key difference in the guarantees below (compared to the ones in Lemma~\ref{lem:witness-producing-robust-moment-estimation}) is the bound on the Frobenius (instead of the weaker spectral) distance between the estimated covariance and true unknown covariance. 

\begin{lemma}[Guarantees for Witness-Producing Robust Moment Estimation Algorithm] \label{lem:witness-producing-robust-moment-estimation-gaussian}
Given a subset of of $n$ points $Y \subseteq \bbQ^d$ whose entries have bit complexity $B$, Algorithm~\ref{algo:robust-moment-estimation-witness-producing} runs in time $(Bn)^{O(1)}$ and either (a.) outputs ``reject,'' or (b.) returns a sequence of weights $0 \leq p_1, p_2, \ldots, p_n$ satisfying $p_1 + p_2 + \cdots + p_n = 1$. 

Moreover, if there exists a set $X \subseteq \R^d$ of points with $4$-certifiably $C$-hypercontractive degree $2$ polynomials with mean $\mu_*$, covariance $\Sigma_*$, then \cref{algo:robust-moment-estimation-witness-producing} does not reject, and the corresponding estimates $\hat{\mu} = \frac{1}{n} \sum_i p_i y_i$ and $\hat{\Sigma} = \sum_{i = 1}^n p_i (y_i -\hat{\mu})(y_i - \hat{\mu})^{\top}$ satisfy the following guarantees:
\begin{enumerate}
   \item \textbf{Mean Estimation: } \[
    \forall u \in \R^d, \text{  } \iprod{\hat{\mu}-\mu_*,u} \leq O(\sqrt{C}) \eta^{3/4} \sqrt{ u^{\top} \Sigma_*u}\mcom
    \]
    \item \textbf{Covariance Estimation: }
    \[
     \Norm{\Sigma_*^{-1/2} \hat{\Sigma} \Sigma_*^{-1/2}-I}_F \leq O(C\eta^{1/2})  \mcom
    \]
    \item \textbf{Witness: } For $C' \leq C(1 + O(\eta^{1/2}))$,
    \[
   \sststile{}{Q} \Set{ \frac{1}{n} \sum_{i=1}^n p_i \Paren{\iprod{y_i - \hat{\mu},Q(y_i - \hat{\mu})}-\frac{1}{n} \sum_{i=1}^n p_i \iprod{y_i - \hat{\mu},Q(y_i - \hat{\mu})}}^2 \leq C' \Norm{\hat{\Sigma}^{1/2}Q\hat{\Sigma}^{1/2}}_F^2 }
    \]
\end{enumerate}
\end{lemma}

We can now use the above witness-producing algorithm to obtain a stronger Frobenius norm estimation guarantee with $(\epsilon,\delta)$-privacy for Gaussian distributions. Notice that the only change from the previous section is in the choice of the constraint system $\cA$ and the corresponding change in the witness checking step. 

\begin{mdframed}
      \begin{algorithm}[Private Robust Moment Estimation]
        \label{algo:private-robust-hypercontractive}\mbox{}
        \begin{description}
        \item[Given:]
        A set of points $Y = \{y_1, y_2, \ldots, y_n\} \subseteq \bbQ^d$, parameters $\eta, \epsilon,\delta > 0$, $L \in \N$.
        \item[Output:]
        Estimates $\hat{\mu}$ and $\hat{\Sigma}$ for mean and covariance. 
        \item[Operation:]\mbox{}
        \begin{enumerate}
         \item  \textbf{Stable Outlier Rate Selection: } Use the $(\eps/3, \delta/3)$-DP \SelectionAlg with $\kappa = L/2$ to sample an integer $\tau \in [\eta n]$ with the scoring function as defined in~\Cref{def:em-score}. If $\tau = \perp$, then reject and halt. Otherwise, let $\eta' = \tau / n$. \label{step:outlier-rate-selection-hypercontractive}
          \item \textbf{Witness Checking: }Compute a pseudo-distribution $\tzeta$ of degree $O(1)$ satisfying $\cA$ on input $Y$ with outlier rate $\eta'$ and minimizing $\Pot_{\eta',\tzeta}(Y)$. Let $\gamma \sim \tLap\left(-\left(1 + \frac{3\ln\left(3/\delta\right)}{\eps}\right),3/\eps\right)$. Check that the weight vector $p=\pE_{\tzeta}[w]$ induces a distribution on $Y$ that has $(C+\gamma)$-certifiably hypercontractive polynomials. If not, reject immediately. Otherwise, let $\widetilde{\mu} = \pE_{\tzeta}[\mu]$ and $\widetilde{\Sigma} = \pE_{\tzeta}[\Sigma]$.
          \item \textbf{Noise Addition: } Let $\gamma_1 = O(C') (L/n)^{\frac{1}{4}}$ and $\gamma_2 = O(C' ) (L/n)^{\frac{1}{4}}$. Let $z \sim \cN(0,\sigma_1)^d$ and $Z \sim \cN(0,\sigma_2)^{{d+1}\choose 2}$, where we interpret $Z$ has a symmetric $d \times d$ matrix with independent lower-triangular entries, and $\sigma_j = 12\epsilon^{-1}\gamma_j \sqrt{2\ln(15/\delta)}$ for $1\leq j\leq 2$. Then, output:
          \begin{itemize}
           \item $\hat{\mu} = \widetilde{\mu} + \widetilde{\Sigma}^{1/2} z$.
           \item $\hat{\Sigma} = \widetilde{\Sigma} + \widetilde{\Sigma}^{1/2} Z \widetilde{\Sigma}^{1/2}$.
          \end{itemize}
        \end{enumerate}

        \end{description}
      \end{algorithm}
\end{mdframed} 

The parameter closeness from potential stability is also upgraded from \cref{cor:param-closeness}:

\begin{lemma}[Parameter Closeness from Stability of Potential]
Let $\eta, \epsilon, \delta > 0$ and $L\in \N$ be given input parameters to \cref{algo:private-robust-hypercontractive} such that $0.25\eta n \geq L = \Omega\left(\frac{1}{\epsilon}\cdot \log\left(\frac{n}{\beta\delta}\right)\right)$. Also, let $Y,Y'$ be adjacent subsets of $\bbQ^d$. 
Suppose \Alg does not reject in any of the 3 steps, uses the constant $C'$ in Step 2 and chooses $\eta'$ in Step~\ref{step:outlier-rate-selection} on input $Y$ and $Y'$. 

Then, for every $u \in \R^d$ and $\theta = \sqrt{L/n}$, we have:
\[
\iprod{\mu_p-\mu_{p'},u} \leq O(C')\theta^{3/4} \sqrt{u^{\top} \Sigma_p u}
\]
and
\[
\Norm{\Sigma_p^{-1/2} \Sigma_{p'} \Sigma_p^{-1/2} - I}_F \leq O(C') \theta^{1/2} \mper
\]
\end{lemma}

The following theorem summarizes our privacy and utility guarantees for the algorithm above. We specialize to the ``base case assumption'' of $4$-certifiable $C$-hypercontractivity of degree $2$ polynomials in order to derive explicit bounds here. Our analysis of the algorithm above follows \emph{mutatis mutandis} with the key upgrade being the stronger Frobenius norm guarantees in Lemma~\ref{lem:witness-check-succeeds} that hold under certifiably hypercontractivity constraints in our constraint system $\cA$ (this requires us to use a version of \cref{lem:hockeysticktensored} that makes use of a bound on $\Norm{AA^T-I}_F$ instead of $\Norm{AA^T-I}_2$; see the remark at the end of \cref{lem:hockeysticktensored}). As before, the $\widetilde{\Omega}$ notation hides logarthmic multiplicative factors in $d$, $C$, $1/\eta$, $1/\epsilon$, and $\ln(1/\delta)$.

\begin{theorem}[Private Robust Mean and Covariance Estimation for Certifiably Hypercontractive Distributions] \label{thm:diff-priv-robust-moment-estimation-hypercontractive}
Fix $C_0 > 0$. Then, there exists an $\eta_0 > 0$ such that for any given outlier rate $0 < \eta \leq \eta_0$ and $\epsilon,\delta > 0$, there exists a randomized algorithm $\Alg$ that takes an input of $n \geq n_0 = \widetilde{\Omega}\left(\frac{d^8}{\eta^2} \left(1 + \frac{\ln(1/\delta)}{\epsilon}\right)^4 \cdot C^4\right)$ points $Y = \{y_1, y_2, \dots, y_n\} \subseteq \bbQ^d$ (where $C = C_0 + \frac{3 \ln(3/\delta)}{\epsilon} + \frac{9}{\epsilon} + 1$), runs in time $(Bn)^{O(1)}$ (where $B$ is the bit complexity of the entries of $Y$) and outputs either ``reject'' or estimates $\hat{\mu} \in \bbQ^d$ and $\hat{\Sigma} \in \bbQ^{d \times d}$ with the following guarantees:
\begin{enumerate}
    \item \textbf{Privacy: } $\Alg$ is $(\epsilon,\delta)$-differentially private with respect to the input $Y$, viewed as a $d$-dimensional database of $n$ individuals. 
    \item \textbf{Utility: } Suppose there exists a $4$-certifiably $C_0$-subgaussian set $X = \{x_1, x_2, \dots, x_n\} \subseteq \bbQ^d$ such that $|Y \cap X| \geq (1-\eta_0)n$ with mean $\mu_*$ and covariance $\Sigma_* \succeq 2^{-\poly(d)}I$. Then, with probability at least $9/10$ over the random choices of the algorithm, $\Alg$ outputs estimates $\hat{\mu} \in \bbQ^d$ and $\hat{\Sigma}\in\bbQ^{d \times d}$ satisfying the following guarantees:
    \[
    \forall u \in \R^d, \text{  } \iprod{\hat{\mu}-\mu_*,u} \leq O(\sqrt{C}\eta^{3/4})  \sqrt{ u^{\top} \Sigma_*u}\mcom
    \]
    and,
    \[
    \Norm{\Sigma_*^{-1/2} \hat{\Sigma} \Sigma_*^{-1/2}-I}_F \preceq O(C \sqrt{\eta})   \mper
    \]
 
\end{enumerate}
Moreover, the algorithm succeeds (i.e., does not reject) with probability at least $9/10$ over the random choices of the algorithm.

\end{theorem}

When specialized to Gaussian distributions, the Frobenius guarantee above is suboptimal---the robust estimation algorithms of~\cite{DKKLMS16} allow estimating the mean and covariance of the unknown Gaussian distribution to an error $\widetilde{O}(\eta)$. We can in fact recover the stronger guarantees by relyong on the analysis in~\cite{kothari2021polynomialtime}[Theorem 1 and 2] of the same constraint system above for the case of Gaussian distributions (in the ``utility case''). This yields the following corollary:



\gaussianest*

\bibliographystyle{alpha}
\bibliography{bib/allrefs,bib/ref,bib/custom2,bib/custom,bib/dblp,bib/scholar,bib/mathreview}

\appendix
\section{Missing Proofs from \Cref{subsec:prelim-dp}}

\subsection{Proof of \Cref{lem:tlap-dp}}
\label{app:tlap}

\begin{proof}[Proof of \Cref{lem:tlap-dp}]
Consider any neighboring datasets $Y, Y'$ and let $\cM$ denote the truncated Laplace mechanism (with parameter as specified). Let $p, q$ denote the probability density functions of $\cM(Y), \cM(Y')$. Observe that $p(x) \leq e^\eps \cdot q(x)$ for all $x < \min\{f(Y), f(Y')\}$. Thus, we have
\begin{align*}
D_{e^\epsilon}(p,q) &= \int_{x \in \R} [p(x)-e^{\epsilon}q(x)]_+ dx \\
&= \int_{x \geq \min\{f(Y), f(Y')\}} [p(x)-e^{\epsilon}q(x)]_+ dx \\
&\leq \int_{x \geq \min\{f(Y), f(Y')\}} p(x) dx \\
(\text{Since sensitivity of } f \text{ is at most } \Delta) &\leq \int_{x \geq f(Y) - \Delta} p(x) dx \\
(\text{\Cref{lem:tlap-tail}}) &\leq \delta,
\end{align*}
which means that the truncated Laplace mechanism is indeed $(\eps, \delta)$-DP.
\end{proof}

\subsection{Proof of \Cref{lem:composition}}
\label{app:composition}

The proof of the composition lemma follows from that of the standard adaptive composition of approximate DP proof~\cite[Theorem 16]{DworkL09}.
Below we use the notation $[x]_+$ to denote $\max\{x, 0\}$ and $x \wedge y$ to denote $\min\{x, y\}$.

\begin{proof}[Proof of \Cref{lem:composition}]
It suffices to prove the theorem for $k = 2$ as we may then apply induction to arrive at the statement for any positive integer $k$. To prove the case $k = 2$, consider any $S \subseteq O_2 \cup \{\perp\}$ and any pair of neighboring datasets $Y, Y'$.

For any $S_1 \subseteq \cO_1 \cup \{\perp\}$, we define the measure $\mu(S_1) := [\Pr[\cM_1(Y) \in S_1] - e^{\eps_1} \Pr[\cM_1(Y') \in S_1]]_+$. Note that we have $\mu(\cO_1) \leq \delta_1$ due to our assumption that $\cM_1$ is $(\eps_1, \delta_1)$-DP. 


Now consider four cases:
\begin{itemize}
\item \textbf{Both $Y, Y'$ satisfy $\Psi_1$.} In this case, we may appeal to $(\eps_2, \delta_2)$-DP under $\Psi_1$ of $\cM_2$ which implies
\begin{align} \label{eq:a2-dp}
\Pr[\cM_2(o_1, Y) \in S] \leq (e^{\eps_2} \Pr[\cM_2(o_1, Y') \in S] \wedge 1) + \delta_2.
\end{align}
For ease of notation, let $p_Y: \cO_1 \to \R^+$ denote the measure obtained by restricting the probability density function of $\cM_1(Y)$ to $\cO_1$ (note that $\int_{\cO_1} p_Y(o_1)\,do_1 = 1 - \Pr[\cM_1(Y) = \perp]$). Then, observe that
\begin{align*}
\Pr[\cM(Y) \in S] 
&= \ind[\perp \in S] \Pr[\cM_1(Y) =\perp] + \int_{\cO_1} \Pr[\cM_2(o_1, Y) \in S] p_Y(o_1)\, do_1 \\
&\overset{\eqref{eq:a2-dp}}{\leq} \ind[\perp \in S] \Pr[\cM_1(Y) = \perp] + \int_{\cO_1} \left((e^{\eps_2} \Pr[\cM_2(o_1, Y') \in S] \wedge 1) + \delta_2\right) p_Y(o_1)\, do_1 \\
&\leq \ind[\perp \in S] \Pr[\cM_1(Y) = \perp] + \delta_2 \\ &\qquad +  \int_{\cO_1} (e^{\eps_2} \Pr[\cM_2(o_1, Y') \in S] \wedge 1) p_Y(o_1)\,do_1 \\
&\leq \ind[\perp \in S] (e^{\eps_1}\Pr[\cM_1(Y') = \perp] + \mu(\{\perp\})) + \delta_2 \\ &\qquad +  \int_{\cO_1} (e^{\eps_2} \Pr[\cM_2(o_1, Y') \in S] \wedge 1) (e^{\eps_1} p_{Y'}(o_1)\,do_1 + d\mu(o_1)) \\
&\leq  \ind[\perp \in S] (e^{\eps_1}\Pr[\cM_1(Y') = \perp]) + \mu(\cO_1 \cup \{\perp\}) + \delta_2\\ &\qquad +  \int_{\cO_1} (e^{\eps_2} \Pr[\cM_2(o_1, Y') \in S] \wedge 1) (e^{\eps_1} p_{Y'}(o_1))\,do_1 \\
&\leq \ind[\perp \in S] (e^{\eps_1 + \eps_2}\Pr[\cM_1(Y') = \perp]) + \delta_1 + \delta_2\\ &\qquad +  \int_{\cO_1} e^{\eps_1 + \eps_2} \Pr[\cM_2(o_1, Y') \in S] p_{Y'}(o_1)\,do_1 \\
&\leq \delta_1 + \delta_2 + e^{\eps_1 + \eps_2}\Pr[\cM(Y') \in S].
\end{align*}
\item \textbf{$Y$ satisfies $\Psi_1$ but $Y'$ does not.} In this case, we have $\Pr[\cM(Y') = \perp] = 1$, which implies that
\begin{align*}
\Pr[\cM(Y) \in S] - e^{\eps_1 + \eps_2} \Pr[\cM(Y') \in S]
&\leq \Pr[\cM(Y) \neq \perp] = \Pr[\cM(Y) \neq \perp] - e^{\eps_1} \Pr[\cM(Y') \neq \perp] \leq \delta_1,
\end{align*}
where the last inequality follows from the fact that $\cM_1$ is $(\eps_1, \delta_1)$-DP.
\item \textbf{$Y'$ satisfies $\Psi_1$ but $Y$ does not.} In this case, we have $\Pr[\cM(Y) = \perp] = 1$, which implies that
\begin{align*}
\Pr[\cM(Y) \in S] - e^{\eps_1 + \eps_2} \Pr[\cM(Y') \in S]
&\leq [\Pr[\cM(Y) = \perp] - e^{\eps_1 + \eps_2} \Pr[\cM(Y') = \perp]]_+ \\
&\leq [\Pr[\cM(Y) = \perp] - e^{\eps_1} \Pr[\cM(Y') = \perp]]_+ \\
&\leq \delta_1,
\end{align*}
where the last inequality once again follows from the fact that $\cM_1$ is $(\eps_1, \delta_1)$-DP.
\item Neither $Y$ nor $Y'$ satisfy $\Psi_1$. In this case, both $\cM(Y)$ and $\cM(Y')$ always output $\perp$. Therefore, we have $\Pr[\cM(Y) \in S] = \Pr[\cM(Y') \in S]$.
\end{itemize}
Thus, in all cases, we have $\Pr[\cM(Y) \in S] = e^{\eps_1+ \eps_2} \Pr[\cM(Y') \in S] + \delta_1 + \delta_2$ as desired.
\end{proof}

\subsection{Proof of \Cref{lem:hockeytriangle}}
\label{app:hockeytriangle}

\begin{proof}[Proof of \Cref{lem:hockeytriangle}]
 Then, note that
 \begin{align*}
  D_{e^\epsilon}(p,r) &= \int_{x\in\R^d} [p(x) - e^\epsilon r(x)]_+\,dx\\
  &= \int_{x\in\R^d} [(p(x) - e^{\epsilon/2} q(x)) + (e^{\epsilon/2} q(x) - e^\epsilon r(x))]_+\,dx\\
  &\leq \int_{x\in\R^d} [(p(x) - e^{\epsilon/2} q(x))]_+\,dx + \int_{x\in\R^d} [e^{\epsilon/2} q(x) - e^\epsilon r(x))]_+\,dx\\
  &= \int_{x\in\R^d} [(p(x) - e^{\epsilon/2} q(x))]_+\,dx + e^{\epsilon/2} \int_{x\in\R^d} [q(x) - e^{\epsilon/2} r(x))]_+\,dx\\
  &= D_{e^{\epsilon/2}}(p,q) + e^{\epsilon/2} \cdot  D_{e^{\epsilon/2}}(q,r),
 \end{align*}
as desired.
\end{proof}

\subsection{Proof of \Cref{thm:dp-apx-selection}}
\label{app:selection}

As stated earlier, the proof of \Cref{thm:dp-apx-selection} follows from applying the exponential mechanism~\cite{McSherryT07} and then use the truncated Laplace mechanism (\Cref{lem:tlap-dp}) to check that the score indeed exceeds $\kappa$.

\begin{proof}[Proof of \Cref{thm:dp-apx-selection}]
\SelectionAlg works as follows:
\begin{enumerate}
\item First, run the $(\eps/2)$-DP exponential mechanism~\cite{McSherryT07}, i.e. selecting each $c \in C$ with probability proportional to $\exp\left(\frac{\eps}{4\Delta} \cdot \score(c, Y)\right)$. Let $c_1$ be the output of this procedure.
\item Sample the noise $N \sim \tLap\left(-\Delta\left(1 + \frac{2\ln\left(1/\delta\right)}{\eps}\right), \frac{2\Delta}{\eps}\right)$ and compute $\widetilde{\score} = \score(c_1, Y) + N$. If $\widetilde{\score} \geq \kappa$, then output $c_1$. Otherwise, output $\perp$.
\end{enumerate}

We will now prove each of the claimed properties:
\begin{enumerate}
\item The first step satisfies $(\eps/2)$-DP via the standard privacy guarantee of the exponential mechanism~\cite{McSherryT07}. The second step is $(\eps/2, \delta)$-DP due to~\Cref{lem:tlap-dp}. Thereby, applying the basic composition theorem implies that \SelectionAlg is $(\eps,\delta)$-DP.
\item Since $N \leq 0$, we are guarantee that if the algorithm outputs $c^* \in \cC$, we must have $\score(c, Y) \geq \kappa$ as desired.
\item For any $c \in \cC$, the standard utility analysis of the exponential mechanism~\cite{McSherryT07} implies that, with probability $1 - 0.5\beta$, we have $\score(c_1, Y) \geq \score(c, Y) - O\left(\frac{\Delta}{\eps}\ln\left(\frac{|C|}{\beta}\right)\right)$. Moreover, the tail bound of Laplace noise (\Cref{lem:tlap-tail}) implies that with probability $1 - 0.5\beta$ we have $N \geq -\Delta\left(1 + \frac{\ln\left(1/\delta\right)}{\eps}\right) - O\left(\frac{\Delta}{\eps} \ln(1/\beta)\right) \geq -O\left(\frac{\Delta}{\eps}\ln\left(\frac{1}{\delta \beta}\right)\right)$. Therefore, if $\score(c, Y) \geq \kappa + O\left(\frac{\Delta}{\eps}\ln\left(\frac{|C|}{\delta \beta}\right)\right)$, the probability that the algorithm outputs $\perp$ is at most $\beta$, as desired.
\end{enumerate} 
\end{proof}

\end{document}